\def\year{2021}\relax
\documentclass[letterpaper]{article} 
\usepackage{aaai21}  
\usepackage{times}  
\usepackage{helvet} 
\usepackage{courier}  
\usepackage[hyphens]{url}  
\usepackage{graphicx} 
\usepackage{natbib}  
\usepackage{caption} 
\usepackage{bm}
\usepackage{ifthen}
\usepackage{mathrsfs}
\usepackage{graphicx}
\usepackage[switch]{lineno}  %

\usepackage{subfigure}

\usepackage{amsmath}
\usepackage{amsthm}
\usepackage{algorithmic}
\usepackage{algorithm}
\usepackage{amssymb}

\newtheorem{remark}{Remark}
\newtheorem{theorem}{Theorem}

\newtheorem{lemma}{Lemma}

\newcommand{\E}{\mathbb{E}}

\newcommand{\argmax}{\operatornamewithlimits{argmax}}
\newcommand{\Var}{\operatornamewithlimits{Var}}

\newcommand{\I}{\mathbb{I}}

\title{Adaptive Algorithms for Multi-armed Bandit with Composite and Anonymous Feedback}
\author {
        Siwei Wang\textsuperscript{\rm 1},
        Haoyun Wang\textsuperscript{\rm 2},
        Longbo Huang\textsuperscript{\rm 1} \\
}
\affiliations {
    \textsuperscript{\rm 1}Tsinghua University \\
    \textsuperscript{\rm 2}Georgia Institute of Technology \\
    wangsw2020@mail.tsinghua.edu.cn, hwang800@gatech.edu, longbohuang@tsinghua.edu.cn \\
}
\begin{document}
\maketitle

\begin{abstract}
We study the multi-armed bandit (MAB) problem with composite and anonymous feedback. In this model, 
the reward of pulling an arm spreads over a period of time (we call this period as reward interval) and the player receives partial rewards of the action, convoluted with rewards from pulling other arms, successively. 
Existing results on this model require prior knowledge about the reward interval size as an input to their algorithms. 
In this paper, we propose adaptive algorithms for both the stochastic and the adversarial cases, without requiring any prior information about the reward interval. 
For the stochastic case, we prove that our algorithm guarantees a regret  that matches the lower bounds (in order). 
For the adversarial case, 
we propose the first algorithm to jointly handle non-oblivious adversary and unknown reward interval size. 
We also conduct simulations based on real-world dataset. The results show that  our algorithms outperform existing benchmarks. 
\end{abstract}
\section{Introduction}
\label{Intro}

The multi-armed bandit (MAB) model \cite{Berry1985Bandit,Sutton1998Reinforcement} has found wide applications in Internet services, e.g., \cite{ck2018ijcai,chapelle2015simple,chen2013combinatorial,jain2018firing,wang2018multi}, 
and attracts increasing attention. 
The classic MAB model can be described as a time-slotted game between the environment and a player. In each time slot, the player has $N$ actions (or arms) to choose from. After he makes the selection, the player receives a reward from the chosen arm immediately. The rewards can be independent random variables generated from certain unknown distributions,  known as the stochastic MAB problem \cite{Lai1985Asymptotically}, or arbitrarily chosen by the environment, called the adversarial MAB problem \cite{Auer2002The}. 
In both models, the player's goal is to maximize his expected cumulative reward during the game by choosing arms properly. To evaluate the player's performance, the concept of ``regret'',  defined as the expected gap between the player's total reward and offline optimal reward, is introduced as the evaluation metric.  


Most of existing MAB algorithms, e.g., UCB  \cite{gittins1989multi,Auer2002Finite} and EXP3 \cite{Auer2002The}, are based on the fact that feedback for the arm-pulling action can be observed precisely and immediately. 
However, in many real-world applications, it is common that \emph{arm rewards spread over an interval and are convoluted with each other.}  
As a concrete example, consider a company conducting advertisements via Internet. The effect of an advertisement, i.e., how it affects the number of clicks (reward), can often spread over the next few days after it is displayed. 
Specifically, in the next couple of days, there will be continuous clicks affected by the advertisement. 
Moreover, during this time, the company usually launches some other ads, which may also impact the number of clicks. 
As a result, the company only observes aggregated information on the reward (thus the feedback is also anonymous).
This situation also happens in medical problems. For example, recent research found that the variability of blood glucose level is the key in controlling diabetes \cite{Hirsch2005Should}. Yet, diabetes medicines do not cause sudden jumps on the blood glucose level. Instead, their effects last for a period, and the blood glucose level is often jointly affected by medicines taken within a period. 
These two features make it hard to separate effects of different medicines, as well as to estimate their effectiveness. 


To address the above difficulties, in this paper, we consider an MAB model where the reward in each time slot is a positive vector.
%
Specifically, the reward vector from pulling arm $i$ at time $t$ is $\bm{r}_i(t)=(r_{i, \tau}(t), \tau \ge 1)$, where $r_{i, \tau}(t)$ denotes the reward component in time $t+\tau$ from pulling arm $i$ at time $t$.  
In addition, at time $t$, the player cannot observe each individual $r_{i, \tau}(t)$ directly.  Instead, he observes the aggregated reward, i.e.,  
$\sum_{\tau\ge 1} r_{a(\tau), t - \tau}(\tau)$, where $a(\tau)$ represents the chosen arm at time step $\tau$.
%
%

%
Existing solution for this problem is to group time slots into rounds \cite{pikeburke2018bandits,Bianchi2018Nonstochastic}, and choose to pull  only  \emph{one} arm in each round. 
With a proper selection of the round size and other parameters, the problem can be connected to the non-anonymous setting \cite{neu2010online,joulani2013online}. 
%
However, these algorithms crucially rely on the precise knowledge of the reward interval size (or the delay of rewards). As a result, underestimating the interval size leads to no theoretical guarantees, whereas overestimation worsens the performances, since their regret bounds are positively related to this estimation. 
Because of this, there is a potential complication of estimation when the prior knowledge is inaccurate, making the algorithms sensitive and less robust. 

To deal with this challenge,
%
in this paper we remove the requirement of any prior knowledge about the reward interval size. 
This is motivated by the fact that, in practical, e.g., medical applications, such information can be unknown or hard to obtain exactly. 
To solve the problem, we propose adaptive methods with increasing round sizes, to mitigate the influence of the reward spread and convolution and improve learning. 
Note that since we do not possess information about the reward interval size, it is critical and challenging to properly choose the speed for round size increase. 
%
%
%
%
Our analysis shows that, with a proper round size increasing rate (which does not depend on the knowledge of the reward interval size), our adaptive policies always possess theoretical guarantees on regrets in both the stochastic case and the adversarial case.
Our main contributions are summarized as follows:
\begin{enumerate}
\item  
We consider the stochastic MAB model with composite and anonymous feedback, where each arm's reward  spreads over a period of time. Under this model, we propose the ARS-UCB algorithm,  which requires \emph{zero} a-prior knowledge about the reward interval size. 
We show that ARS-UCB achieves an $O(N\log T + c(d_1,d_2,N))$ regret, where $c(d_1,d_2,N)$ is a function that does not depend on $T$, and $d_1$ and $d_2$ are measures of the expectation and variance of the composite rewards, respectively. 
Our regret upper bound matches  the regret lower bound for this problem, as well as regret bounds of existing policies that require knowing the exact reward interval size. 
%

\item
We propose the ARS-EXP3 algorithm for the adversarial MAB problem with composite and anonymous feedback studied in \cite{Bianchi2018Nonstochastic}. ARS-EXP3 does not require any knowledge about the reward interval size, and works in the case where the delays are non-oblivious. 
We show that ARS-EXP3 achieves an $O((d+(N\log N)^{1\over 2})T^{2\over 3})$ regret, where $d$ is the size of the reward interval. 
To the best of our knowledge, ARS-EXP3 is the first efficient algorithm in this setting (i.e., where the delays are non-oblivious).

%

\item We conduct extensive experiments 
based on real-world datasets, to validate our theoretical findings. 
The results are consistent with our analysis, and show that our algorithms 
outperform state-of-the-art benchmarks. Thus, our adaptive policies are more robust and can be used more widely in real applications.
\end{enumerate}

\subsection{Related Works}

Stochastic MAB with delayed feedback is first proposed in \cite{joulani2013online,agarwal2011distributed,desautels2012parallelizing}. In \cite{joulani2013online}, the authors propose a BOLD framework to solve this problem. In this framework, the player only changes his decision when a feedback arrives. Then, decision making can be done the same as with non-delayed feedback. They show that the regret of BOLD can be upper bounded by $O(N(\log T + \E[d]))$, where $d$ represents the random variable of delay. 
\cite{manegueu2020stochastic} then explored the case that the delay in each time slot is not i.i.d., but depends on the chosen arm. In this setting, they proposed the PatientBandits policy, which achieves near optimal regret upper bound. 
In addition to the stochastic case, adversarial MAB with delayed feedback also attracts people's attention. This model is first studied in \cite{weinberger2002on}, where it is assumed that the player has full  feedback. The paper establishes a regret lower bound of $\Omega(\sqrt{(d+1)T\log N})$ for this model, where $d$ is a constant feedback delay. The model with bandit feedback is investigated in \cite{neu2010online,neu2014online}, where the authors used the BOLD framework \cite{joulani2013online} to obtain a regret upper bound of $O(\sqrt{(d+1)TN})$. 
%
Recently, 
\cite{zhou2019learning, thune2019nonstochastic, bistritz2019online} made more optimizations on MAB with delayed 
feedback. Since their analytical methods are used in the non-anonymous setting, they are very different and cannot be used for our purpose.

\cite{pikeburke2018bandits} extends the model to contain anonymous feedback, and gives a learning policy called ODAAF. ODAAF uses  information of the delay as inputs, including its mean and variance. This helps the algorithm to estimate the upper confidence bounds.
The regret upper bound of ODAAF is $O(N(\log T + \E[d]))$, which is the same as BOLD with non-anonymous feedback. 
\cite{garg2019stochastic} then explores the composite and anonymous feedback setting and makes some minor changes to generalize ODAAF policy. However, their algorithm still needs to use precise knowledge of the reward interval. 
%
As for regret lower bound, \cite{Vernade2010Stochastic} generalizes the  regret lower bound of classic MAB model. They show that the stochastic MAB problem with delayed feedback still has a regret lower bound $O(N\log T)$. To the best of our knowledge, there is no known regret lower bound for the MAB model with delayed and anonymous feedback. 
We thus use $O(N\log T)$ as a regret lower bound in this model to compare our results with. 

Inspired by the stochastic setting, \cite{Bianchi2018Nonstochastic} studied the adversarial MAB model with composite and anonymous feedback, and present the CLW algorithm to solve the problem. 
In their paper, the losses (or the rewards) are assumed to be oblivious, so that the environment cannot change them during the game. 
They obtain a regret upper bound $O(\sqrt{dTN})$ for the CLW algorithm, and establish a matching $\Omega(\sqrt{dTN})$ regret lower bound. 
%


\section{Stochastic MAB with Composite and Anonymous Rewards}\label{Section_Sto}

We start with the stochastic case and first introduce our model setting in Section \ref{sub_section_sto_model}. 
Then, we present our Adaptive Round-Size UCB (ARS-UCB) algorithm and its regret upper bound with a proof sketch in Section \ref{sub_section_sto_algorithm}. 
Due to space limit, we put the complete proofs in the appendix.

\subsection{Model Setting}\label{sub_section_sto_model}

We adapt the model setting in \cite{garg2019stochastic}, and allow the reward intervals to have infinite size. Specifically, in our setting, a player plays a game 
for $T$ time slots. 
%
%
In each time slot, the player chooses one arm among a set of $N$ arms $\mathcal{N}=\{1, \cdots, N\}$ to play. Each arm $i$, if played, generates an i.i.d. reward vector in $\mathbb{R}_+^\infty$, where 
$\mathbb{R}_+$ is the set of all non-negative real numbers.\footnote{We allow the effects of pulling an arm to last forever, while prior works, e.g., \cite{garg2019stochastic}, all assume a finite reward interval size.
} 
We denote $\bm{r}_{a(t)}(t) = (r_{a(t), 1}(t), r_{a(t), 2}(t), \cdots )$ the reward vector generated by pulling arm $a(t)\in\mathcal{N}$ at time $t$, where the $\tau$-th term $r_{a(t), \tau}(t)$ is the \emph{partial} reward that the player obtains from arm $a(t)$ at time $t+\tau$ after pulling it at time $t$, and without loss of generality, we assume that $||\bm{r}_{a(t)}(t)||_1 \in [0,1]$. 
We denote $D_{a(t)}$ the distribution of $\bm{r}_{a(t)}$ and $\bm{\mu}_{a(t)}\triangleq \E_{D_{a(t)}}[\bm{r}_{a(t)}]$ its mean. %
Then, at every time $t$, the player receives the \emph{aggregated} reward from all previously pulled arms, 
i.e., 
$Y(t) \triangleq \sum_{\tau \le t-1} r_{a(\tau), t-\tau}(\tau)$.
%

Under this model, the expected total reward of pulling arm $i$ is $s_i \triangleq ||\bm{\mu}_i||_1$. 
%
Without loss of generality, we assume $1\ge s_1 > s_2 \ge \cdots \ge s_N \ge 0$, and denote $\Delta_i \triangleq s_1 - s_i$ for all $i\ge 2$ the reward gap of arm $i$. Then, the cumulative regret of the player can be expressed  as  $Reg(T) \triangleq Ts_1 - \E[\sum_{t=1}^T s_{a(t)}]$. 
The goal of the player is to find an algorithm to minimize his $Reg(T)$.

\subsection{ARS-UCB Algorithm}\label{sub_section_sto_algorithm}


To explain the idea of ARS-UCB algorithm (which is presented in Algorithm \ref{Algorithm_UCB}), we first introduce some notations. We denote $N_i(t)$ the number of times the player chooses to pull arm $i$ up to time $t$, and $M_i(t)$ the cumulative observed reward (w.r.t. $Y(t)$) up to $t$ from pulling arm $i$, i.e., $N_i(t) \triangleq \sum_{\tau\le t} \I[a(\tau) = i]$ and 
$M_i(t) \triangleq \sum_{\tau\le t} \I[a(\tau) = i]Y(\tau).$ 
We also denote 
\begin{eqnarray}\label{eq:sit-def} 
\hat{s}_i(t) \triangleq {M_i(t)\over N_i(t)}, 
\end{eqnarray}
 the empirical mean of arm $i$, and define an unknown reward 
$L_i(t) \triangleq  \sum_{\tau\le t} \I[a(\tau) = i]||\bm{r}_i(\tau)||_1$, 
which is the actual cumulative gain from arm $i$ until time $t$ (note that $L_i(t)$ is different from $M_i(t)$). 
If the $L_i(t)$ value for each  arm $i$ is known, then the origin UCB policy can be directly applied with empirical mean ${L_i(t)\over N_i(t)}$'s, and achieve a regret $O(\sum_{i=2}^N {1\over \Delta_i} \log T)$. However, since the player can only observe $M_i(t)$, in order to achieve a good performance, we want to ensure that the difference between $M_i(t)$ and $L_i(t)$ is small. More precisely, as $t$ goes to infinity, we want the difference between ${M_i(t)\over N_i(t)}$ and ${L_i(t)\over N_i(t)}$ to converge to $0$. 
%
\begin{algorithm}[t]
    \centering
    \caption{Adaptive Round-Size UCB (ARS-UCB) }\label{Algorithm_UCB}
    \begin{algorithmic}[1]
    \STATE \textbf{Input: } $f$, $\alpha$.
    \STATE For each arm $i$, play it for $f(1)$ times and set $K_i = 2$.
    \WHILE{\textbf{$t<T$}}
    \STATE For all arm $i$, $u_i(t) = \min \{\hat{s}_i(t) + \sqrt{\alpha\log t\over N_i(t)}, 1\}$, where 
    $\hat{s}_i(t)$ is defined in Eq. \eqref{eq:sit-def}.
    \STATE Play arm $a(t) \in \argmax_i u_i(t)$ for $f(K_{a(t)})$ times (if there are multiple maximum $u_i(t)$, choose the arm with smallest $N_i(t)$).
    \STATE $K_{a(t)} = K_{a(t)} + 1$.
    \ENDWHILE
    \end{algorithmic}
\end{algorithm}

An intuitive approach to achieve this is to choose an increasing function $f: \mathbb{N}_+ \to \mathbb{N}_+$, where $\mathbb{N}_+$ is the set of all positive integers, and to use $f(k)$ as the number of time steps in the $k$-th round. Then, in each round, we only pull a single arm.
Figure \ref{Figure_1} shows the difference between $M_i(t)$ and $L_i(t)$ in each round. In this figure, the rewards in the blue rectangle are the feedback in $M_i(t)$, and the rewards in the red parallelogram are those in $L_i(t)$. We see that no matter how long a round is, the difference is always bounded by the two triangle parts. 
Let $K_i(t)$ be the value of $K_i$ in Algorithm \ref{Algorithm_UCB} until time $t$, and denote $F(K) \triangleq \sum_{k=1}^{K} f(k)$. Then, at the decision time slot in Algorithm \ref{Algorithm_UCB}, $N_i(t) = F(K_i(t))$ for each arm $i$.
On the other hand, Figure \ref{Figure_1} shows  that $|M_i(t) - L_i(t)| = O(K_i(t))$ for each arm $i$.
If $f$ is increasing, $K/F(K)$ will converge to $0$ as $K$ increases. Therefore, $|{M_i(t) \over N_i(t)} - {L_i(t) \over N_i(t)}| = O({K_i(t) \over F(K_i(t))}) \to 0$ as $t$ goes to infinity. As a result, the algorithm behaves like UCB after some time.  

\begin{figure}[t]
\centering
\includegraphics[width=2.2in]{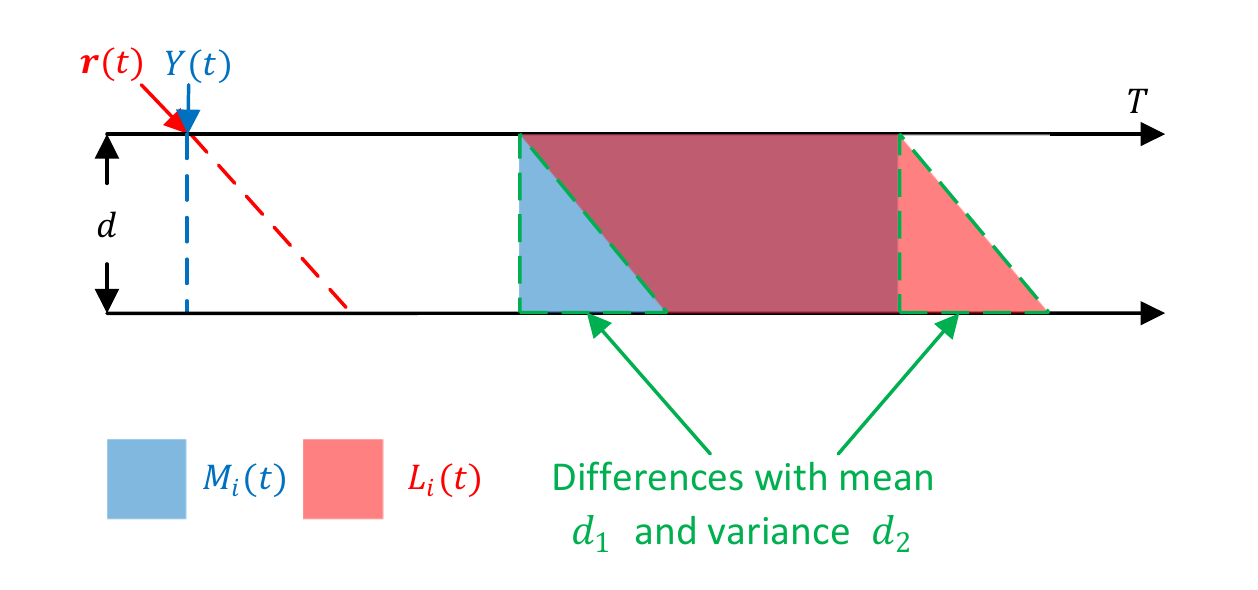}
\caption{The difference between $M_i(t)$ and $L_i(t)$ in each round. It can be bounded by the areas of the two triangles.} 
\label{Figure_1}
\end{figure}

From the above reasoning, we see that the input $f$ in Algorithm \ref{Algorithm_UCB} is introduced to control the convergence rate of ${M_i(t)-L_i(t)\over N_i(t)}$. The other input $\alpha$, used in the confidence radius, is  to control the change of likelihood of the event $\{s_i \le u_i(t), \forall i\}$. Carefully choosing $f$ and $\alpha$ is the key to ensure a good performance of the algorithm.

\begin{theorem}\label{Theorem_UCB}
Suppose $\alpha >  4$ and the function $f$ satisfies (i) $f$ is increasing, and (ii) $\exists k_0$ such that $\forall k > k_0, F(k) \ge f(k+1)$. 
Then,  ARS-UCB achieves that 
	\begin{equation}\label{eq_001}
    Reg(T) \le \sum_{i=2}^N {8\alpha\log T \over \Delta_i} + c_f^*(d_1,d_2,N,\alpha).
\end{equation} 
Here $c_f^*(d_1,d_2,N,\alpha)$ is a constant that does not depend on $T$,\footnote{The term $c_f^*(d_1,d_2,N,\alpha)$ does not depend on $T$ as long as the function $f$ does not depend on $T$, e.g., $f(k) = k^2$ or $f(k) = k^3$.} 
$d_1 \triangleq \sum_{d' = 1}^\infty \max_i \E[\sum_{\tau = d'}^{\infty} r_{i,\tau}]$, $d_2 \triangleq \sum_{d' = 1}^\infty \max_i \Var[\sum_{\tau = d'}^\infty r_{i,\tau}]$, where $r_{i,\tau}$ is the $\tau$-th term of random vector $\bm{r}_i$, and the expectation and variance are taken over the distribution $D_i$. 
\end{theorem}

Notice that it is not hard to find such a function $f$. For example, $f(k) =  ck^\beta $ satisfies the two properties with integers $c \ge 1$ and $\beta \ge 1$. Another example is $f(k) = 2^{k+c}$ with $c \ge 0$ (here we need to set $f(1) = 2^{2+c}$ specifically).
The value $d_1$ in the theorem can be regarded as an upper bound of the expected rewards in the triangle region (in Figure \ref{Figure_1}), and $d_2$ is an upper bound of the variance. Compared to the ODAAF policy in \cite{pikeburke2018bandits,garg2019stochastic}, 
the regret upper bound of ARS-UCB also depends on the mean and variance of the feedback delay. 
However, our algorithm has the advantage that it does not require any prior information about $d_1$ and $d_2$, whereas the ODAAF policy 
takes both $d_1$ and $d_2$ as inputs.  
Thus, ARS-UCB can be applied to settings where such information is not available. 

Another advantage of ARS-UCB is that the constant factor before the $\log T$ term in its regret upper bound is 
much smaller than 
ODAAF. 
This is because that ODAAF follows an elimination structure, and only eliminates a sub-optimal arm when its upper confidence bound is smaller than the \emph{lower} confidence bound of the optimal arm. 
On the other hand, ARS-UCB follows the basic UCB structure, in which the player always chooses the arm with largest upper confidence bound. Since in each time step, there are only tiny changes on upper confidence bounds when $t$ is large,  the upper confidence bounds of sub-optimal arms are approximately equal to the \emph{upper} confidence bound of the optimal arm in the end of the game. 
Therefore, one needs to pull each sub-optimal arm  more in ODAAF to obtain a smaller upper confidence bound (to match the \emph{lower} confidence bound of the optimal arm rather than the \emph{upper} confidence bound). As a result,  a larger regret upper bound occurs. 
%
This fact is also supported by our simulation results, i.e., ARS-UCB always outperforms ODAAF.
%

Lastly, although ARS-UCB chooses a same arm in each round, 
doing so does not cause excessive regret compared to UCB, as ARS-UCB can be viewed as grouping the plays of arms into consecutive intervals. This is also validated in our regret analysis and simulation results. 

\begin{remark}
The performance guarantees for ARS-UCB hold for any constant $\alpha > 4$ and increasing function $f$. However, to avoid a large constant in regret, choosing a function $f$ that increases faster would be better, e.g., $f(k) = k^2$. As for $\alpha$, when the delay measures $d_1,d_2$ are large and $T$ is small, a larger $\alpha$ can reduce the regret. On the other hand, when $d_1,d_2$ are small but $T$ is large, a smaller $\alpha$ behaves better. 
\end{remark}



\begin{proof}[Proof Sketch of Theorem \ref{Theorem_UCB}]

%
%
Note that in classic UCB policy, $s_i \le v_i(t) \triangleq {L_i(t) \over N_i(t)} + \sqrt{4\log t \over N_i(t)}$ with high probability. Thus we want to ensure that for large enough $t$, we have $u_i(t) \ge v_i(t)$, i.e., ${L_i(t) - M_i(t) \over N_i(t)} \le (\sqrt{\alpha} - 2)\sqrt{\log t \over N_i(t)}$ (or equivalently, ${L_i(t) - M_i(t) \over \sqrt{N_i(t)}} \le (\sqrt{\alpha} - 2)\sqrt{\log t}$). If this inequality holds, we know that $s_i \le u_i(t)$ with high probability.

As described in Figure \ref{Figure_1}, in a round $[t_1,t_2]$ such that $a(t) = i$ for all $t\in [t_1,t_2]$, the gap between $\sum_{t = t_1}^{t_2} Y(t)$ and $\sum_{t=t_1} ^{t_2} ||\bm{r}_{a(t)}(t)||$ are the two triangle terms, i.e., 
\begin{eqnarray}
\nonumber\sum_{t = t_1}^{t_2} Y(t) &=& \sum_{t=t_1} ^{t_2} ||\bm{r}_{a(t)}(t)|| + \sum_{t \le t_1-1} \sum_{\tau = t_1-t}^{\infty} r_{a(t),\tau}(t)
\label{eq_111} \\
&& - \sum_{t \le t_2}\sum_{\tau = t_2-t+1}^{\infty} r_{a(t),\tau}(t).
\end{eqnarray}

Summing over the rounds we choose arm $i$, the gap between $M_i(t)$ and $L_i(t)$ is $\Theta(K_i(t))$. 
If $f(k)$ is increasing, we know that ${K_i(t)\over\sqrt{N_i(t)} }\le \sqrt{2}$. Hence, there must be some time step $T^*=c_f^*(d_1,d_2,N,\alpha)$ such that ${L_i(t) - M_i(t) \over N_i(t)} = \Theta({K_i(t) \over N_i(t)}) \le (\sqrt{\alpha} - 2)\sqrt{\log t \over N_i(t)}$ for any $t > T^*$. This means that ARS-UCB is efficient after time step $T^*$, which results in the regret upper bound in Eq. \eqref{eq_001}.
\end{proof}

\section{Non-oblivious Adversarial MAB with Composite and Anonymous  Rewards}

We first introduce the adversarial model setting in Section \ref{sub_section_adv_model}. 
Then, we present our Adaptive Round-Size EXP3 (ARS-EXP3) algorithm for the non-oblivious case and state its regret upper bound in Section \ref{Sub-EXP3}. 
Similarly, a proof sketch is provided, and the complete proofs are referred to the appendix. 
%
%

\subsection{Model Setting}\label{sub_section_adv_model}

In the adversarial MAB model with composite and anonymous feedback, there are $N$ arms $\mathcal{N}=\{1, 2, \cdots, N\}$ and the game lasts for $T$ time steps. 
%
In each time slot $t$, the adversary gives every arm $i$ a reward vector $\bm{r}_i(t) \in \mathbb{R}_+^d$ where $d$ is some unknown constant. 
To normalize the reward, we assume that $||\bm{r}_i(t) ||_1 \le 1$. At any time slot $t$, if the player chooses to pull arm $i$,  he receives reward  $r_{i, \tau}(t)$ at time slot $t+\tau$. 
Similar to the stochastic scenario, in every time slot $t$, the player  receives an  \emph{aggregated} (hence anonymous) reward 
$Z(t) \triangleq \sum_{\tau = t-d}^{t-1} r_{a(\tau), t-\tau}(\tau)$, 
where $a(t)$ represents the chosen arm at time $t$, and $r_{a(t),t-\tau}(\tau)$ is the $(t-\tau)$-th partial reward in $\bm{r}_{a(\tau)}(\tau)$. 
Denote $G_i \triangleq \sum_{t=1}^T ||\bm{r}_i(t)||_1$. The total regret of the player is defined as $Reg(T) \triangleq \E[\max_i G_i] - \E[\sum_{t=1}^T ||\bm{r}_{a(t)}(t)||_1]$. In the following, we also assume for simplicity that  $T$ is known to the player.\footnote{If $T$ is unknown, one can use the doubling-trick method, e.g., in \cite{Lu2010,Sli2014}.}

Note that although our model is the same as the one in  \cite{Bianchi2018Nonstochastic}, we allow the delay to be non-oblivious.
Specifically, for any arm $i$ and time step $t$, the actual reward $s_i(t)$ is pre-determined (i.e., the actual rewards are oblivious). However, the adversary can choose an arbitrary reward vector $\bm{r}_i(t)\in \mathbb{R}_+^d$ based on previous observations, as long as $||\bm{r}_i(t)||_1 = s_i(t)$ (i.e., how the reward spreads over time are non-oblivious). As a result, prior works cannot be applied and it requires new algorithms and analysis. 

%

\subsection{ARS-EXP3 Algorithm} 
\label{Sub-EXP3}
Our algorithm  will similarly use an increasing round size. Given a round size function $g: \mathbb{N}_+ \to \mathbb{N}_+$, the round sizes of the game are set to be $g(1), g(2),\cdots$ Since $T$ is known, we can first compute $K$, the number of all completed rounds during the game, and use $g(K)$ as a  normalization factor.

The algorithm for this adversarial setting is called ARS-EXP3, which is shown in Algorithm \ref{Algorithm_EXP3}. In the algorithm,  the notations $N_i(t) \triangleq \sum_{\tau\le t} \I[a(\tau) = i]$, $M_i(t) \triangleq \sum_{\tau\le t} \I[a(\tau) = i]Z(t)$ and $L_i(t) \triangleq  \sum_{\tau\le t} \I[a(\tau) = i]||\bm{r}_{i}(t)||_1$ remain the same as in the stochastic case. 
In the classic EXP3 policy, the probability of choosing arm $i$ depends on $L_i(t)$ but not ${L_i(t) \over N_i(t)}$. Since the observed values are the $M_i(t)$'s, we need a bound on $|L_i(t) - M_i(t)|$. Yet, we cannot expect $|L_i(t) - M_i(t)|$ to converge to $0$ 
because it can only increase during the game. This means that we cannot use the same analysis as in the stochastic case. 

\begin{algorithm}[t]
    \centering
    \caption{Adaptive Round-Size EXP3 (ARS-EXP3)}\label{Algorithm_EXP3}
    \begin{algorithmic}[1]
    \STATE \textbf{Input: } $g$, $\gamma$, $T$, $w_1 = \cdots = w_N = 1$. 
    \STATE Compute the last round number $K$.
    \FOR {$k = 1,2,\cdots, K$}
    \STATE For any $i$, $e_i = \exp({w_i \over g(K)})$, $p_i = (1-\gamma){e_i \over \sum_i e_i} + {\gamma \over N}$.
    \STATE Draw $a(k) \sim p$, and then pull arm $a(k)$ in round $k$ (with size $g(k)$). Let $Z(k)$ be the collected rewards within this round $k$.
    \STATE $Z'(k) = \min\{Z(k), g(k)\}$, $w_{a(k)} = w_{a(k)} + {\gamma Z'(k)\over N p_{a(k)}}$.
    \ENDFOR
    \end{algorithmic}
\end{algorithm}

Our analysis will be based on the following observation: the regret of using the classic EXP3 policy under our setting can be upper bounded by the regret of using classic EXP3 policy under the classic adversarial MAB model, plus the largest difference  $\max_i |L_i(T)- M_i(T)|$.
The reason is that if we pretend to actually receive reward of $M_i(T)$ from arm $i$, then the regret will be the same as that in the classic model. However, in our model, we receive $L_i(t)$. Thus, our regret upper bound should include the difference term $\max_i |L_i(T)- M_i(T)|$. 
Similar to the stochastic case, the value $|L_i(T)- M_i(T)|$ depends on the number of rounds in the game. 
Hence, we choose an increasing function $g(\cdot)$, to ensure that the algorithm only runs $o(T)$ rounds, so that the regret is sub-linear.

\begin{theorem}\label{Theorem_EXP3}
Set  $\gamma = \min \{1, \sqrt{N\log N \over (e-1)((\beta + 1) T)^{1\over \beta + 1}}\}$ and $g(k) = k^\beta$. Then, Algorithm \ref{Algorithm_EXP3} achieves 
	\begin{equation*}Reg(T) = O((N\log N)^{1\over 2} T^{2\beta + 1\over 2\beta + 2} + dT^{1\over \beta + 1}). \end{equation*}
In particular, if $\beta = {1\over 2}$, Algorithm \ref{Algorithm_EXP3} achieves a regret upper bound $O((d + (N\log N)^{1\over 2}) T^{2\over 3})$.
\end{theorem}

\begin{remark}
Note that the analysis for this case is very different from that in the stochastic case. In the stochastic case, 
as long as the error probability is smaller than ${1\over t^3}$, the round size does not influence the cumulative regret.  In the adversarial case, however, the regret is linear in the largest round size. 
Thus, we need a lower increasing speed. Theorem \ref{Theorem_EXP3}  shows that $g(k) = k^{1\over 2}$ provides a good choice. 
\end{remark}

\begin{proof}[Proof Sketch of Theorem \ref{Theorem_EXP3}] 

Let $K$ be the last completed round until time $T$. Then, there are less than $g(K+1)$ slots left, which can cause  at most $g(K+1)$ additional regret. 

Define $G(K) \triangleq \sum_{k=1}^{K} g(k)$, and consider another game lasting for $G(K)$ time steps, where pulling arm $i$ at time $t$ gives reward $R_i(t) = Z(t) = \sum_{\tau = t-d}^{t-1} r_{a(\tau), t-\tau}(\tau)$.
In this game, Algorithm \ref{Algorithm_EXP3} behaves the same as an EXP3 algorithm running $K$ time steps with largest reward $g(K)$ in each step. These  imply an $O(g(K)\sqrt{NK\log N})$ regret upper bound \cite{Auer2002The}.
Since the cumulative rewards of these two games are the same, the remaining part is the difference between the total rewards of their best arms. 

Similar to the stochastic case, for a round $[t_1,t_2]$ such that $a(t) = i$ for all $t \in [t_1,t_2]$, the following equation \eqref{eq_119} holds.
\begin{eqnarray}
\nonumber\sum_{t = t_1}^{t_2} Z(t) &=& \sum_{t=t_1} ^{t_2} ||\bm{r}_{a(t)}(t)|| + \sum_{t = t_1-d}^{t_1-1} \sum_{\tau = t_1-t}^{d} r_{a(t),\tau}(t)
\label{eq_119} \\
&& - \sum_{t =t_2-d+1}^{t_2}\sum_{\tau = t_2-t+1}^{d} r_{a(t),\tau}(t).
\end{eqnarray}

 This implies that during one round, the difference on the reward of any single arm between the two games can increase by at most $\sum_{t = t_1-d}^{t_1-1} \sum_{\tau = t_1-t}^{d} r_{a(t),\tau}(t)$, which is less than or equal to $d$.
Then, since there are totally $K$ rounds, Algorithm \ref{Algorithm_EXP3} can have an additional regret $Kd$.

Combining the three components, we obtain 
$Reg(T) = O(g(K)\sqrt{NK\log N}+ Kd + g(K+1))$.

When we set $g(k) = k^{\beta}$, then $K = \Theta(T^{1\over \beta + 1})$. Thus, the cumulative regret satisfies that $Reg(T) = O((N\log N)^{1\over 2}T^{2\beta + 1\over 2\beta + 2} + dT^{1\over \beta + 1})$.
\end{proof}


The adversary can choose the reward vectors properly to make sure that every switch between arms causes a constant bias between $M_i(t)$ and $L_i(t)$. This bias makes our observations inaccurate, and is then added to the final regret (the $Kd$ term) according to our analysis. Because of this, our model setting is similar to the non-oblivious adversarial MAB model with switching cost, in which each switch leads to an additional cost. \cite{cesa2013online,dekel2014bandits} show that the non-oblivious adversarial MAB with switching cost has a regret lower bound of $\Omega(T^{2\over 3})$. Therefore, it is reasonable that we can only obtain a similar $O(T^{2\over 3})$ regret upper bound. 



\section{Simulations} 


\subsection{The Stochastic Setting} 
We start with the stochastic case. In our experiments, there are a total of $9$ arms. The expected reward of the 9 arms follows the vector $\bm{s} = [.9,.8,.7,.6,.5,.4,.3,.2,.1]$. We conduct experiments 
on the following cases.
%

\subsubsection{Random delay} 

In this case, the reward of pulling an arm is given to the player after a random delay $z$. 
%
That is, $\forall \tau' \ne z, r_{a(t),\tau'}(t) = 0$ and $\E[r_{a(t), z}(t)] = s_{a(t)}$. We choose $z$ to be i.i.d. uniformly in $[10,30]$ (Figures \ref{Figure_2} and \ref{Figure_3}) or $[0,60]$ (Figures \ref{Figure_4} and \ref{Figure_5}). For comparison, we choose the ODAAF algorithm   proposed in \cite{pikeburke2018bandits} with accurate knowledge about the delay $z$ as benchmark. 

\begin{figure}[h]
\centering 
\subfigure[]{ \label{Figure_2} 
\includegraphics[width=1.34in]{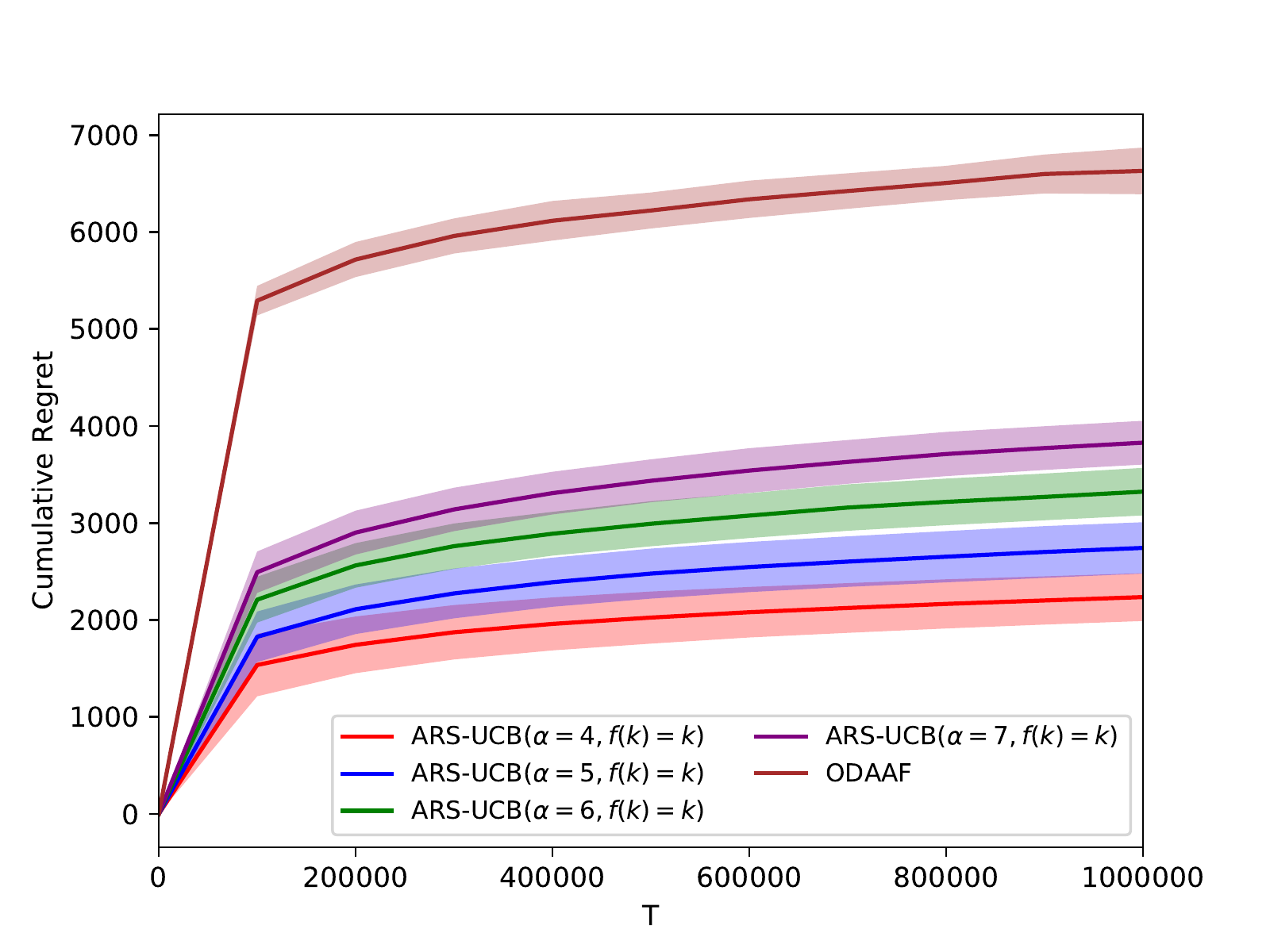}}
\subfigure[]{ \label{Figure_3} 
\includegraphics[width=1.34in]{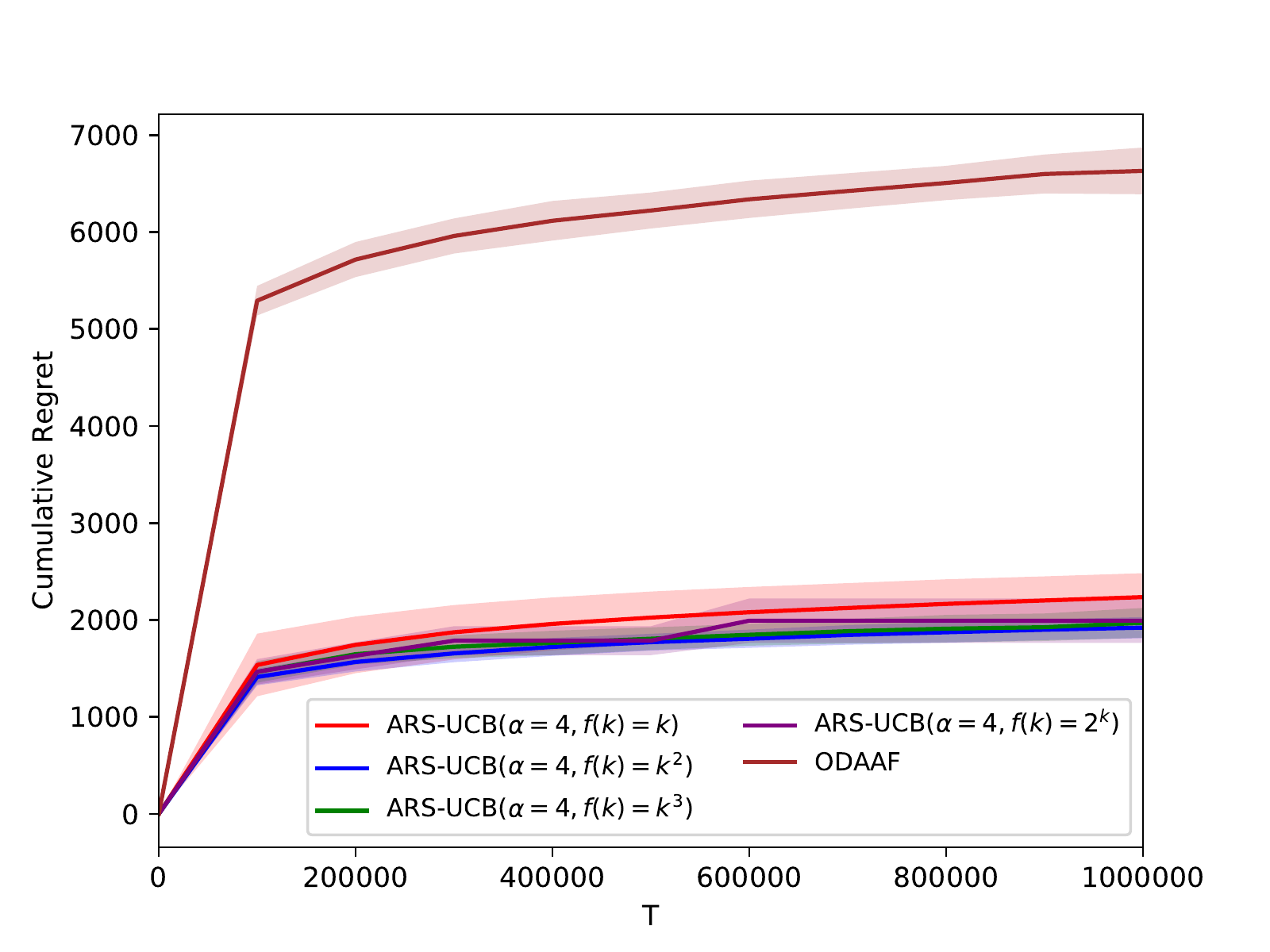}}
\subfigure[]{ \label{Figure_4} 
\includegraphics[width=1.34in]{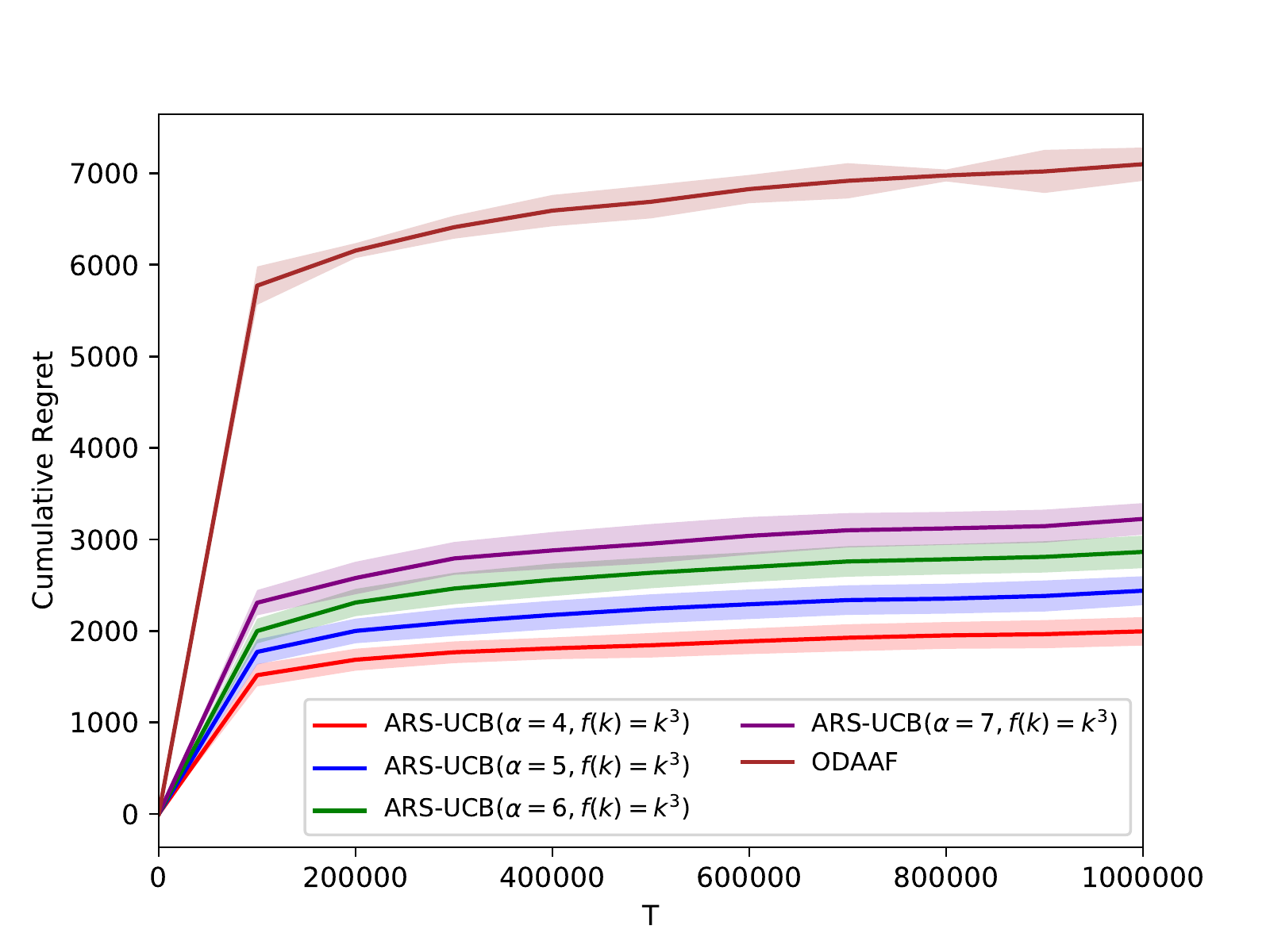}}
\subfigure[]{ \label{Figure_5} 
\includegraphics[width=1.34in]{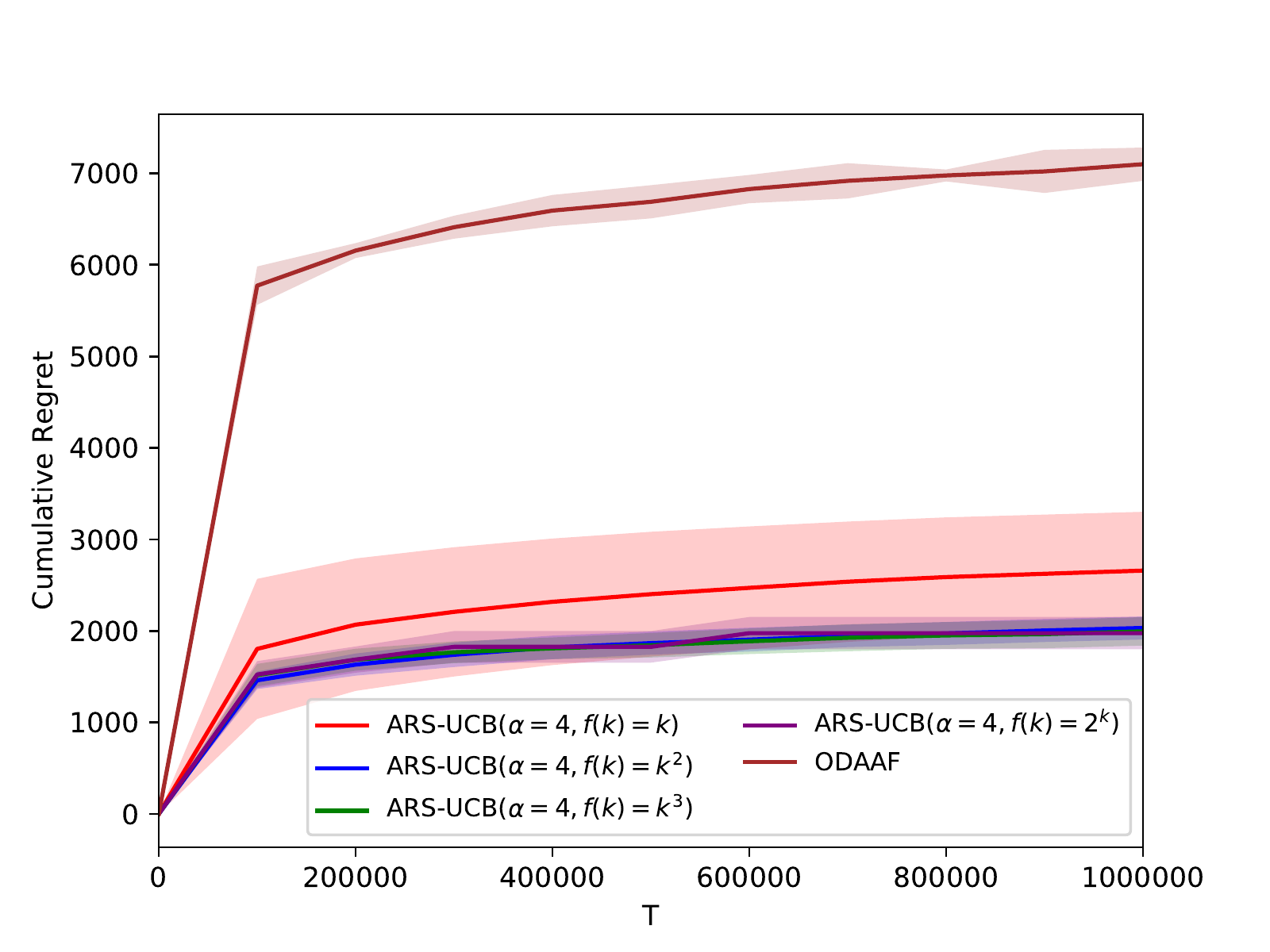}}
\caption{Experiments: Comparison between cumulative regrets of ARS-UCB and ODAAF (delayed reward)
}
\end{figure}

\subsubsection{Bounded interval}
In this case, the reward of pulling an arm at time $t$ takes effect in time interval $[t+d_{\min}, t+d_{\max})$ and the effects within this period remains the same. That is, $\E[r_{a(t), d_{\min}}(t)] = {s_{a(t)}\over d_{\max} - d_{\min}}$, and $\forall \tau \in [t+d_{\min}, t+d_{\max}), r_{a(t), \tau}(t) = r_{a(t), d_{\min}}(t)$ (the reward vectors are of the form $[0,\cdots,0,r,\cdots,r, 0, \cdots]$).  
In Figure \ref{Figure_10} we choose $(d_{\min}, d_{\max}) = (30,40)$, and in Figure \ref{Figure_11} we choose $(d_{\min}, d_{\max}) = (10,20)$.
For comparison, we choose the generalized ODAAF algorithm proposed in \cite{garg2019stochastic} with accurate knowledge about the reward interval size $d_{\max}$  as benchmark. 


\subsubsection{Linearly decreasing reward}
In this case, the reward of pulling an arm at time $t$ takes effect from time $t+1$ and lasts for $d$ time steps. Moreover, its value decreases linearly as time going on. That is, $\forall \tau \in [1,d]$ we have that $\E[r_{a(t),\tau}(t)] = (d+1-\tau) \cdot {2s_{a(t)}\over d(d+1)}$ (the reward vectors are of the form $[dr,(d-1)r,\cdots,r, 0, \cdots]$). In Figure \ref{Figure_12} we choose $d = 100$, and in Figure \ref{Figure_13} we choose $d = 50$. For comparison, we choose the generalized ODAAF algorithm proposed in \cite{garg2019stochastic} with accurate knowledge about the reward interval size $d$  as benchmark. 

\subsubsection{Linearly increasing reward}
In this case, the reward of pulling an arm at time $t$ takes effect from time $t+1$ and lasts for $d$ time steps. Moreover, its value increases linearly as time going on. That is, $\forall \tau \in [1,d]$ we have that $\E[r_{a(t),\tau}(t)] = \tau \cdot {2s_{a(t)}\over d(d+1)}$ (the reward vectors are of the form $[r,2r,\cdots,dr, 0, \cdots]$). In Figure \ref{Figure_18} we choose $d = 100$, and in Figure \ref{Figure_19} we choose $d = 50$. For comparison, we choose the generalized ODAAF algorithm proposed in \cite{garg2019stochastic} with accurate knowledge about the reward interval size $d$  as benchmark. 

\begin{figure}[h]
\centering 
\subfigure[]{ \label{Figure_10} 
\includegraphics[width=1.34in]{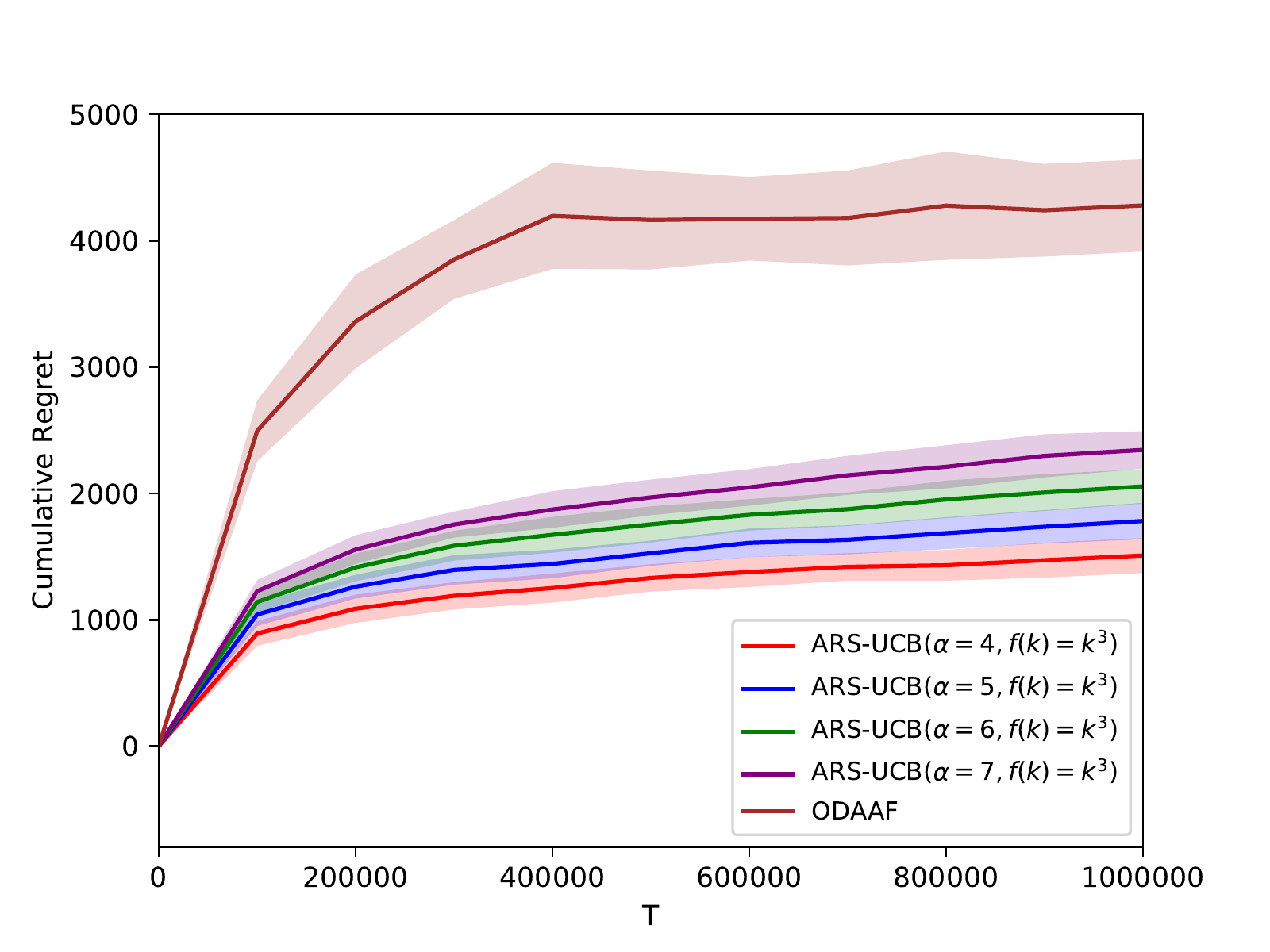}}
\subfigure[]{ \label{Figure_11} 
\includegraphics[width=1.34in]{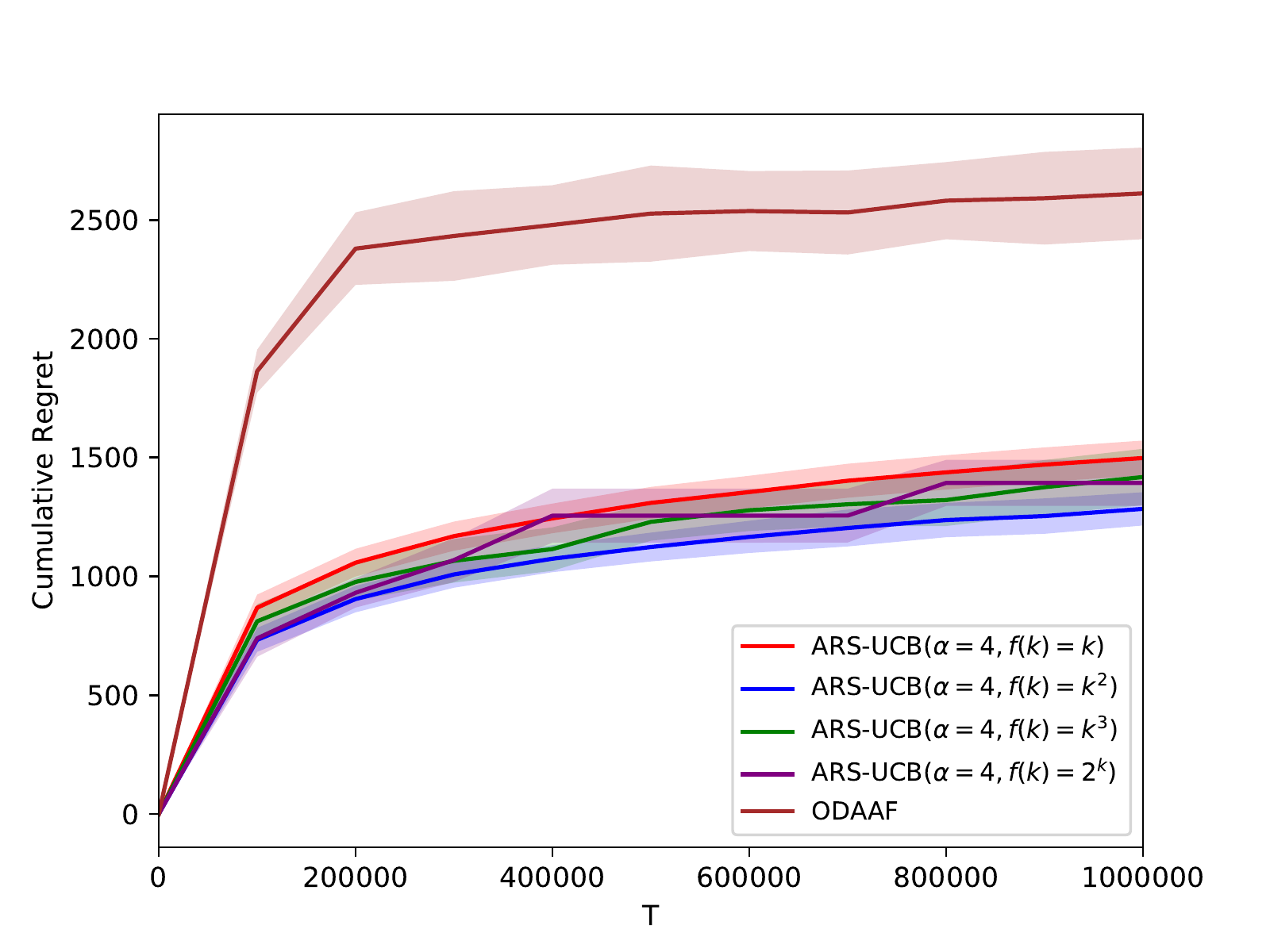}}
\subfigure[]{ \label{Figure_12} 
\includegraphics[width=1.34in]{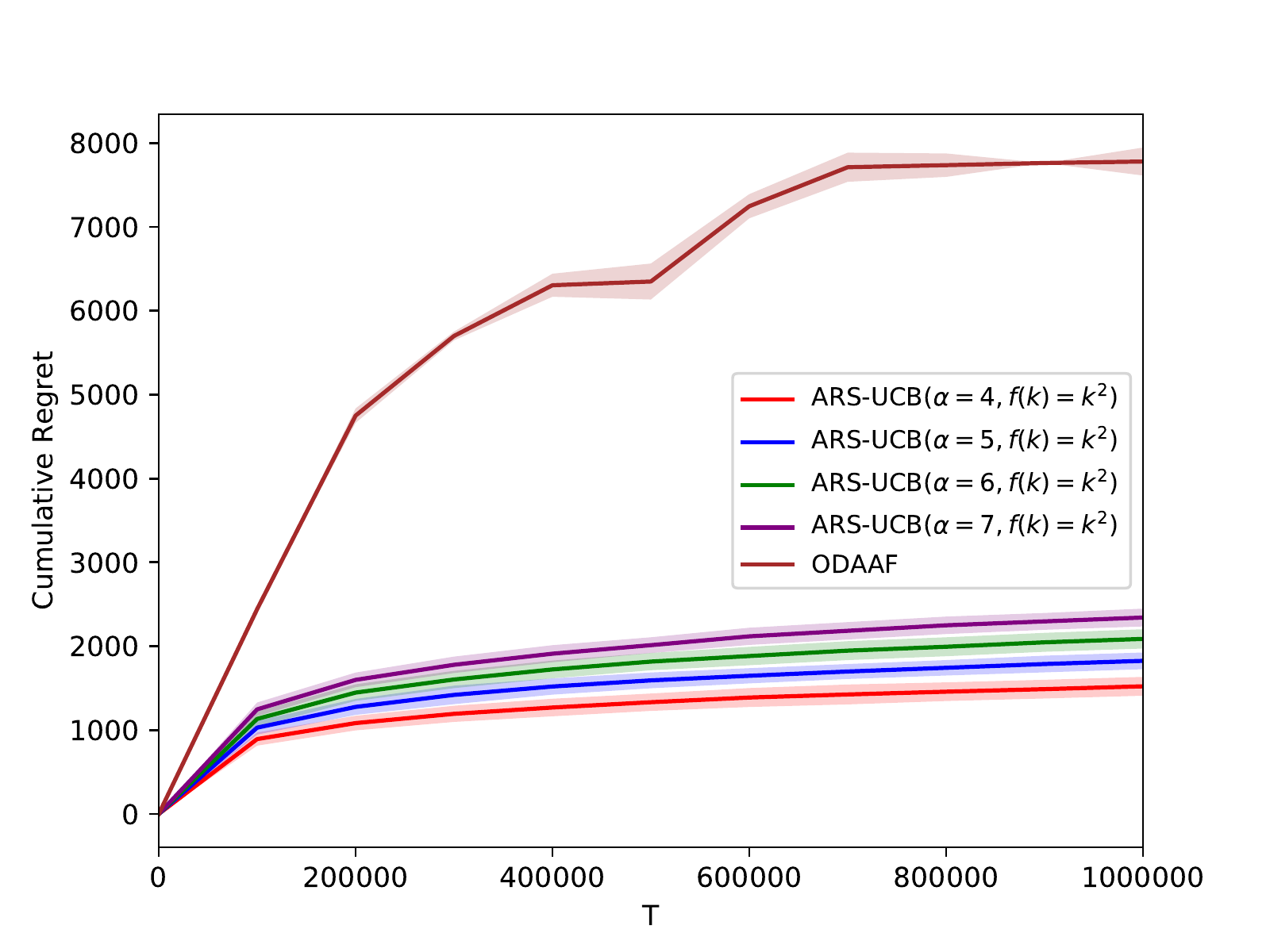}}
\subfigure[]{ \label{Figure_13} 
\includegraphics[width=1.34in]{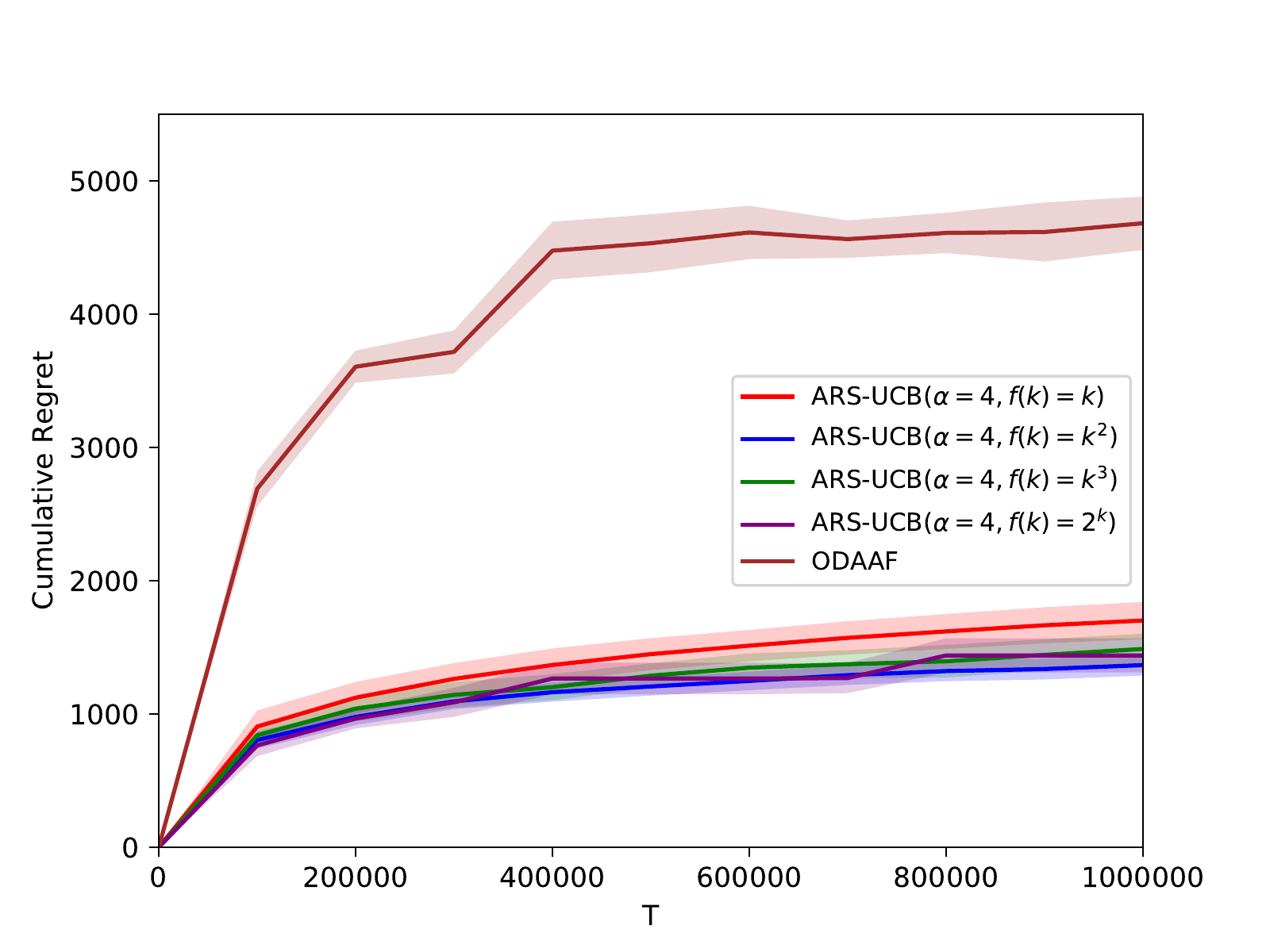}}
\subfigure[]{ \label{Figure_18} 
\includegraphics[width=1.34in]{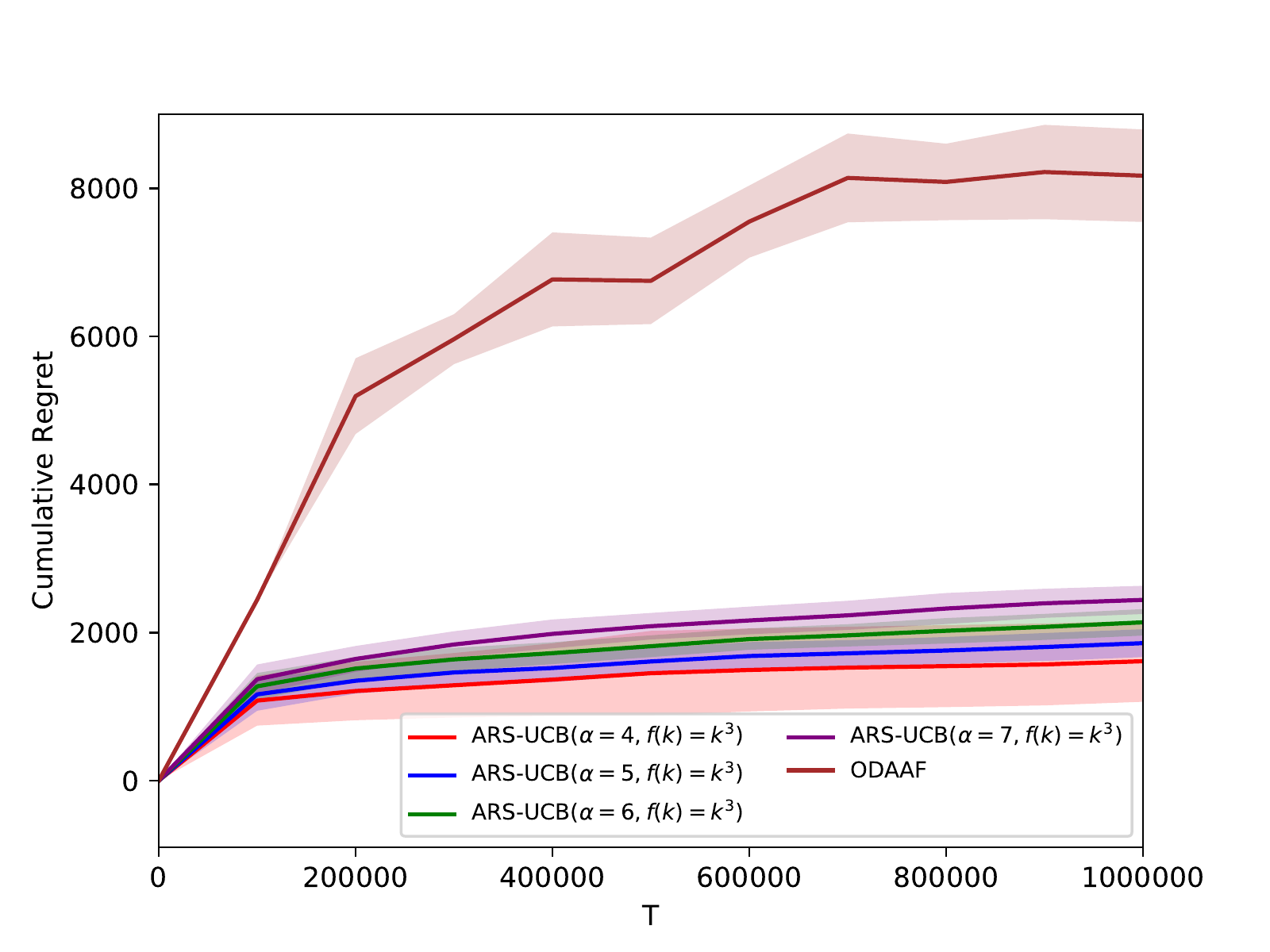}}
\subfigure[]{ \label{Figure_19} 
\includegraphics[width=1.34in]{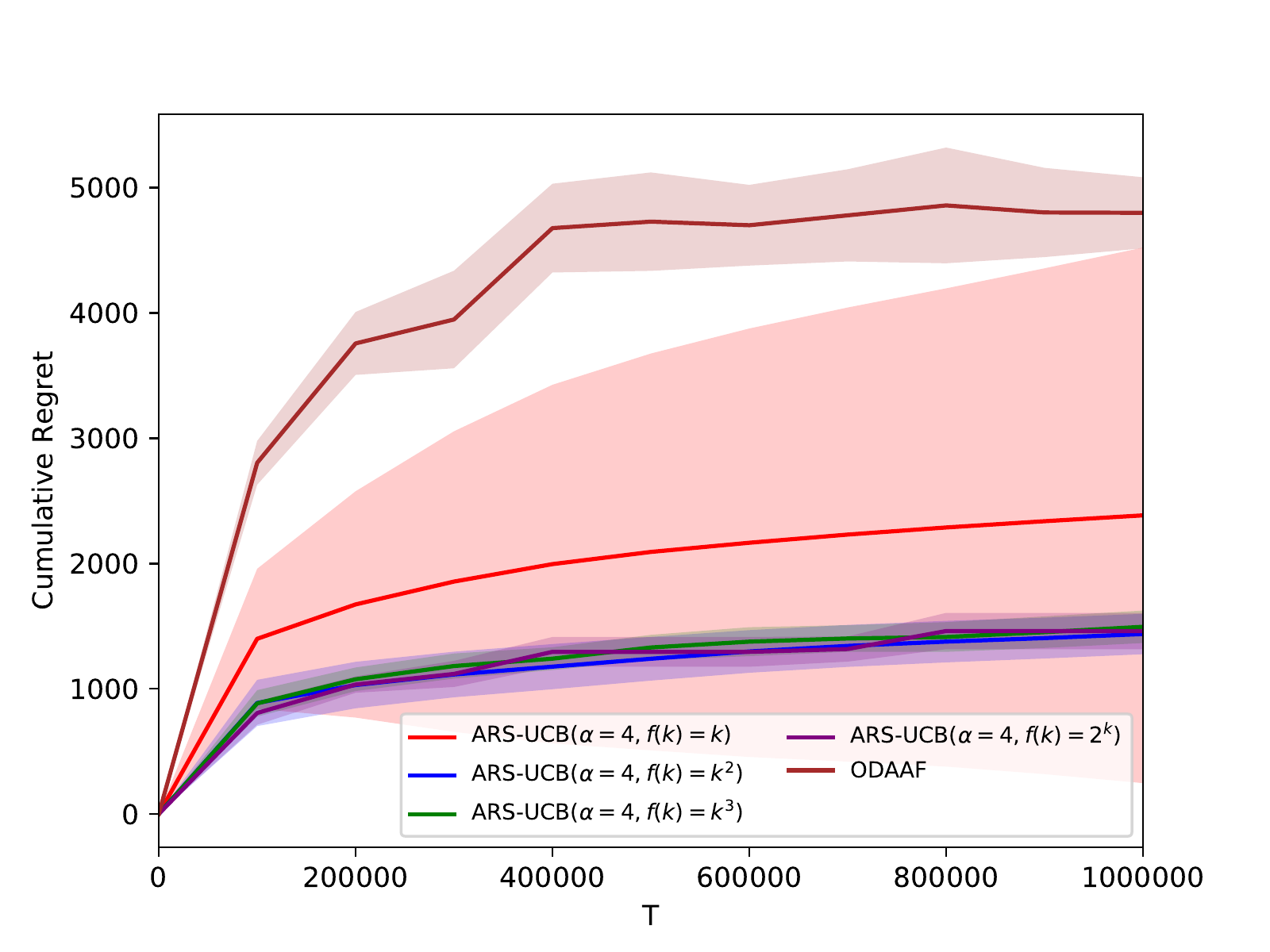}}
\caption{Experiments: Comparison between cumulative regrets of ARS-UCB and ODAAF (composite reward)
}
\label{Figure_Comparison}
\end{figure}


\subsubsection{Discounted reward}
In this case, the reward of pulling an arm at time $t$ takes effect from time $t+1$ and lasts forever. Moreover, its value decreases exponentially with a factor $\gamma \in (0,1)$. That is, $\E[r_{a(t),1}(t)] = (1-\gamma)s_{a(t)}$, and $\forall \tau > 1, r_{a(t),\tau}(t) = \gamma r_{a(t),\tau-1}(t)$ (the reward vectors are of the form $[r, \gamma r, \gamma^2 r, \cdots]$). 
%
In Figure \ref{Figure_14} we choose $\gamma = 0.8$, and in Figure \ref{Figure_15} we choose $\gamma = 0.9$. 

\subsubsection{Polynomially decreasing reward}
In this case, the reward of pulling an arm at time $t$ takes effect from time $t+1$ and lasts forever. Moreover, its value decreases polynomially with a factor $\gamma > 1$. That is, $\forall \tau \ge 1$, $\E[r_{a(t),\tau}(t)] = c_\gamma \cdot s_{a(t)}/ \tau^{\gamma}$, where $c_\gamma$ is a normalization factor such that $\sum_{\tau=1}^\infty \E[r_{a(t),\tau}(t)] = s_{a(t)}$ (the reward vectors are of the form $[r,{r\over 2^\gamma},{r\over 3^\gamma}, \cdots]$).
In Figure \ref{Figure_16}, we choose $\gamma = 3$, and in Figure \ref{Figure_17}, we choose $\gamma = 2$. 

The reward intervals in the discounted reward case and the polynomially decreasing case are with infinite size, and there is no existing benchmarks. Therefore, we only compare the cumulative regrets of ARS-UCB with different parameters.  

\begin{figure}[t]
\centering 
\subfigure[]{ \label{Figure_14} 
\includegraphics[width=1.34in]{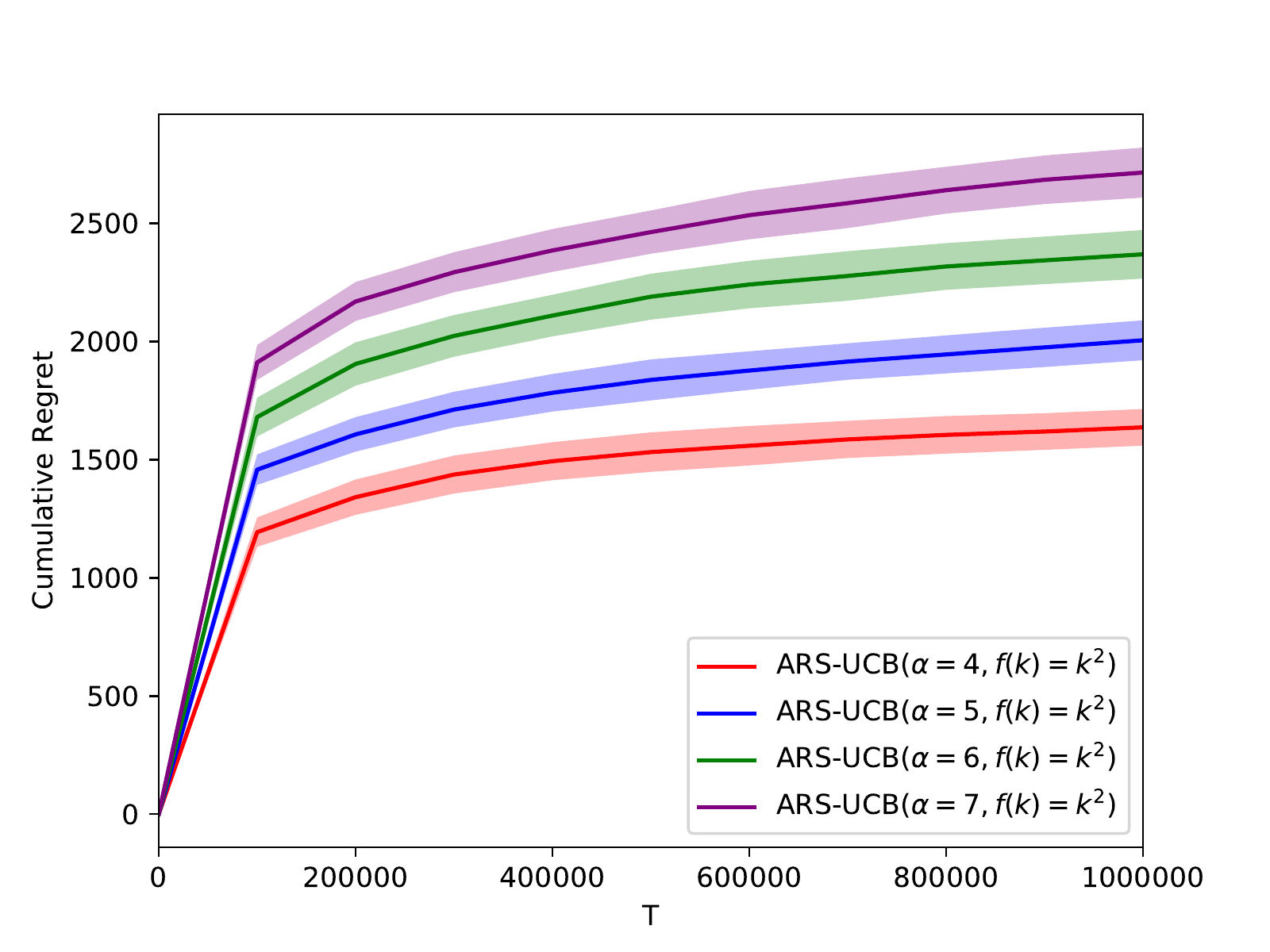}}
\subfigure[]{ \label{Figure_15} 
\includegraphics[width=1.34in]{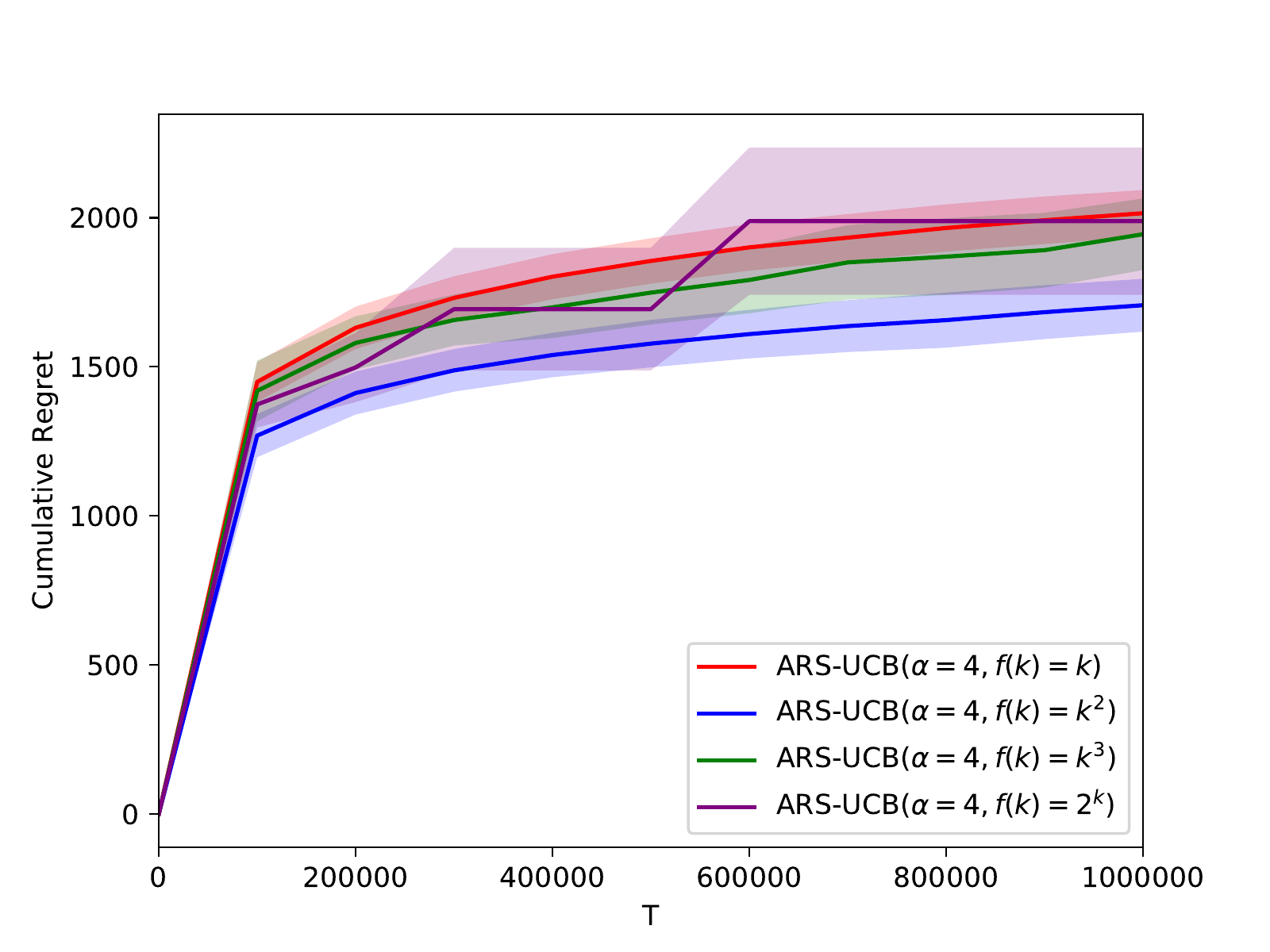}}
\subfigure[]{ \label{Figure_16} 
\includegraphics[width=1.34in]{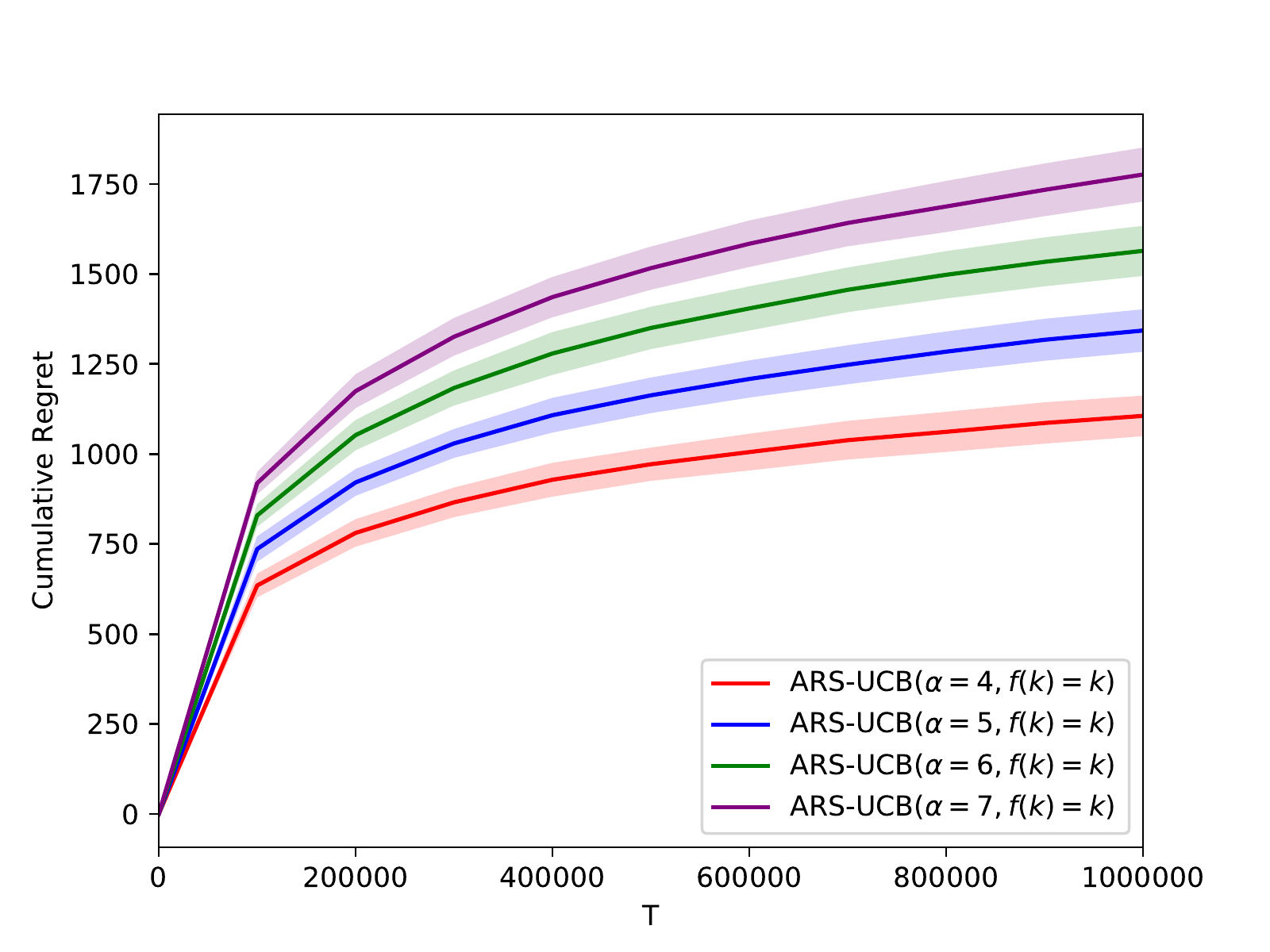}}
\subfigure[]{ \label{Figure_17} 
\includegraphics[width=1.34in]{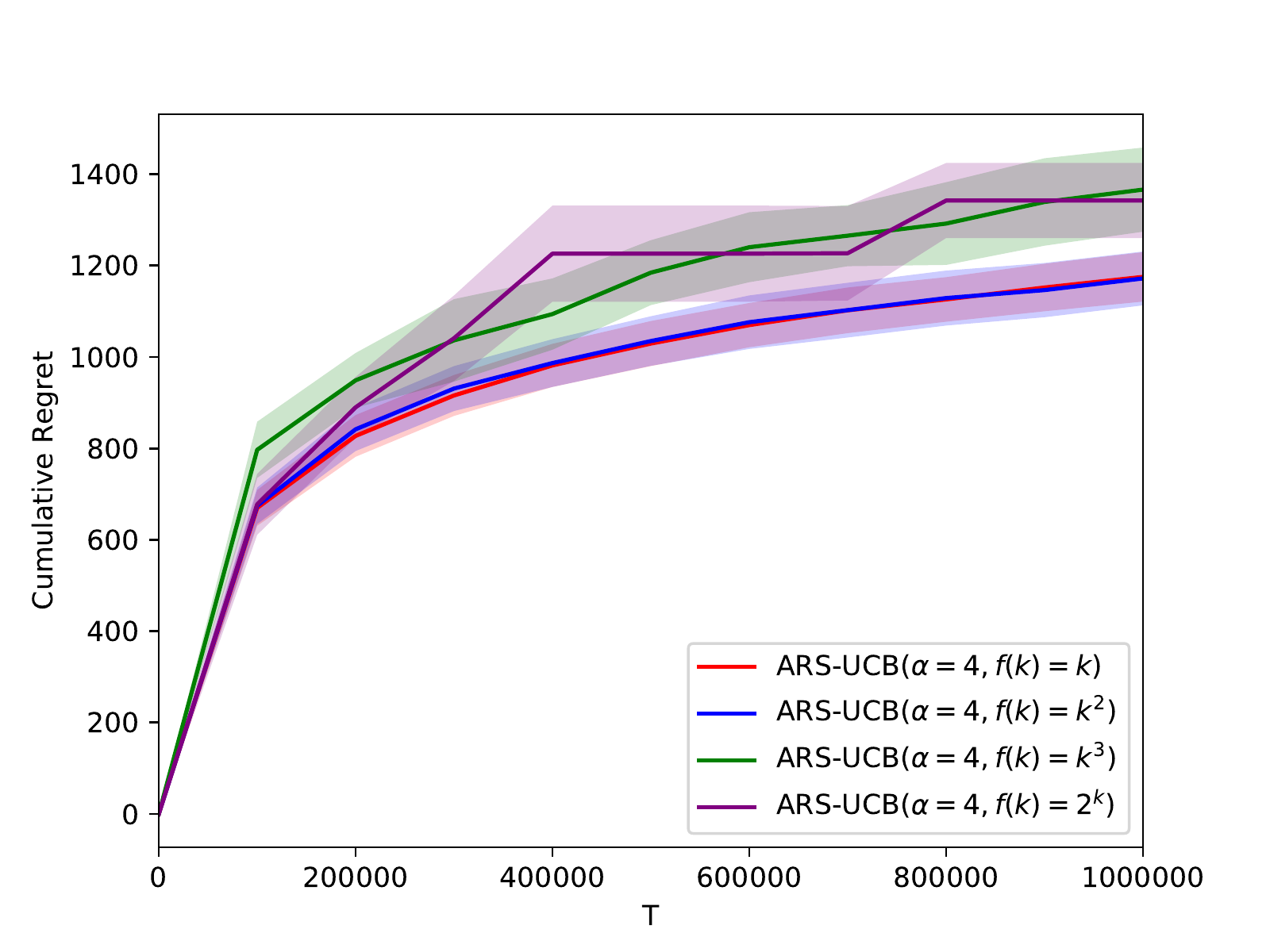}}
\caption{Experiments: The cumulative regrets of ARS-UCB
}
\end{figure}

\subsubsection{Conclusion on experimental results}

In the above experiments, 
we observe that the cumulative regrets of ARS-UCB are always logarithmic in $T$, which is expected from our theoretical analysis. Moreover, in all our experiments, 
ARS-UCB significantly outperforms ODAAF, which assumes full knowledge of the delay size. 
%
%
This is because that in ODAAF 
the player needs to pull  the sub-optimal arms more to eliminate them, which leads to a worse performance than ARS-UCB.  
Therefore, ARS-UCB is more robust and more efficient in the case of minimizing the cumulative regret.

In Figures \ref{Figure_2}, \ref{Figure_4}, \ref{Figure_10}, \ref{Figure_12}, \ref{Figure_18}, \ref{Figure_14} and \ref{Figure_16}, we choose the same function of $f(k)$ and vary the value of $\alpha$, while Figures \ref{Figure_3}, \ref{Figure_5}, \ref{Figure_11}, \ref{Figure_13}, \ref{Figure_19}, \ref{Figure_15} and \ref{Figure_17} show the performance under different $f(k)$ functions with a same value $\alpha=4$ (other $\alpha$ values show similar behavior). From these results, we can see that the combination of $\alpha = 4$ and $f(k) = k^2$ behaves better in most of these problem instances, and leads to both small regrets and small variances.

\subsection{The Adversarial Setting} \label{Sec_Simu_Adv}

%
%
Here we use two datasets in 
Kaggle: 
the \emph{Outbrain Click Prediction}  (Outbrain) dataset \cite{Kaggle1}, and the \emph{Coupon Purchase Prediction} (Coupon) dataset \cite{Kaggle2}. 

The Outbrain dataset records whether users click a provided advertisement when they enter the system.
In this experiment, the system needs to decide the category of advertisements to show to an incoming user, and his goal is to maximize the number of users clicking the advertisement.  
In this dataset, a user only   clicks \emph{one} category of the advertisements. If the system provides the category correctly, it gets one click. Otherwise it gets zero. Thus, the reward (and the feedback) is whether the advertisement is clicked.  
The Coupon dataset records whether users click offered coupons. In this case, a system is providing coupons to its users, and it needs to decide what coupon to offer on each day of the week. 
There are totally seven choices, and only if the user wants to use that coupon on that day, he will choose to click that coupon to see more details.  
The goal of the system is to maximize the number of users that click the coupons. Thus, the reward (and the feedback) of one time slot is whether the coupon is clicked. 
In our setting, the 
feedback is given to the player after a delay $z$, which is artificially simulated. 

\begin{figure}[t]
\centering 
\subfigure[]{ \label{Figure_6} 
\includegraphics[width=1.34in]{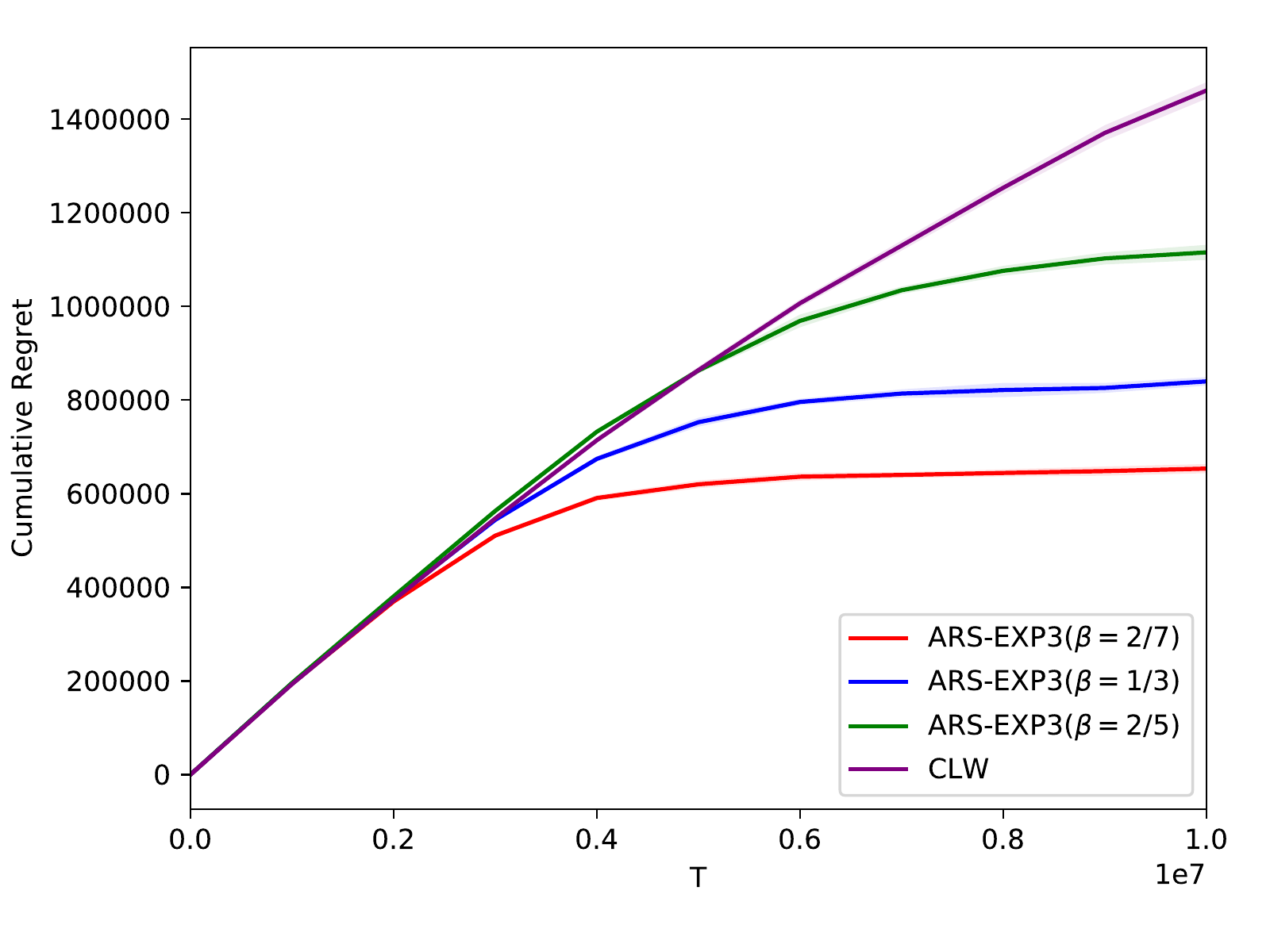}}
\subfigure[]{ \label{Figure_7} 
\includegraphics[width=1.34in]{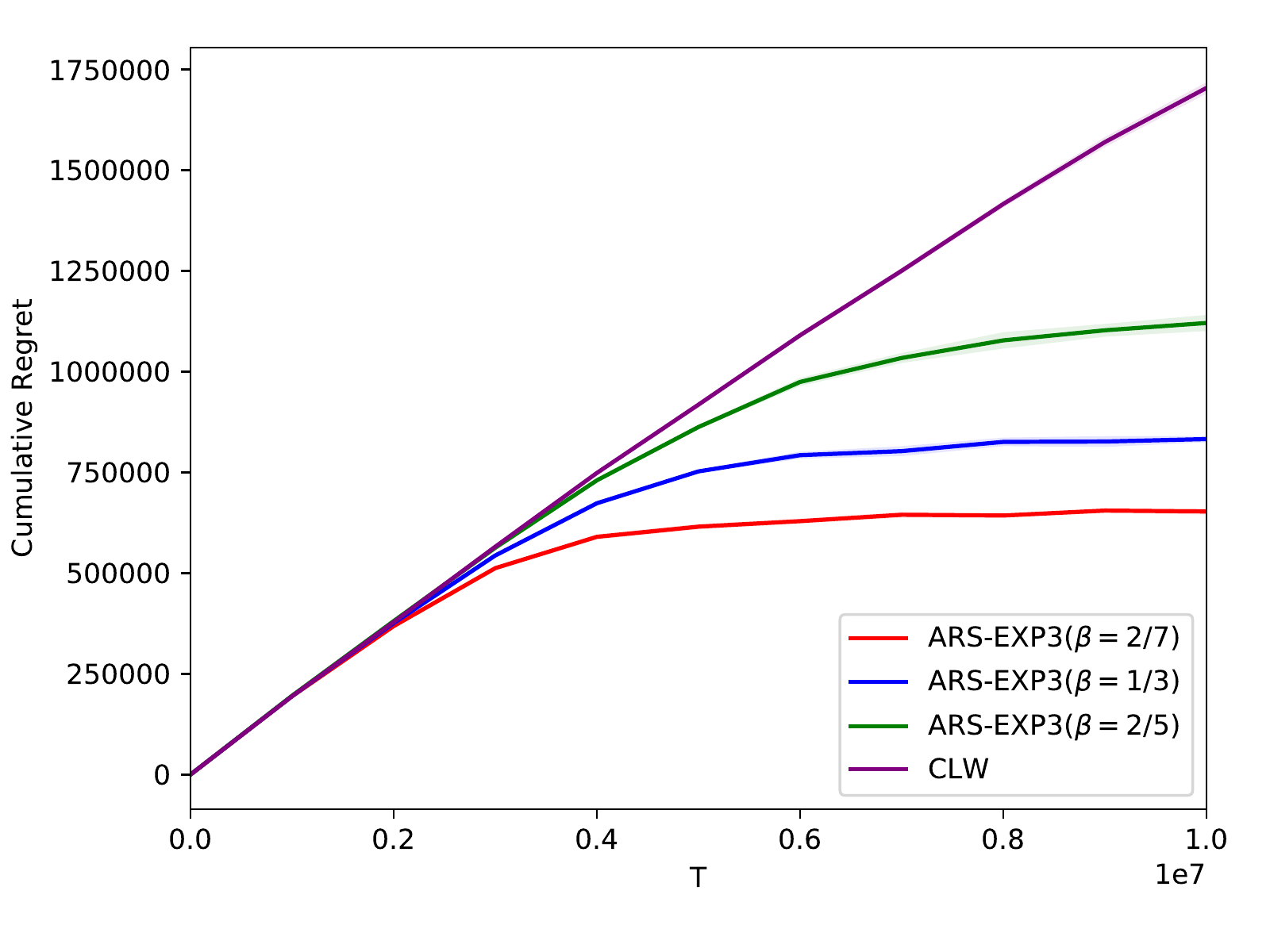}}
\subfigure[]{ \label{Figure_8} 
\includegraphics[width=1.34in]{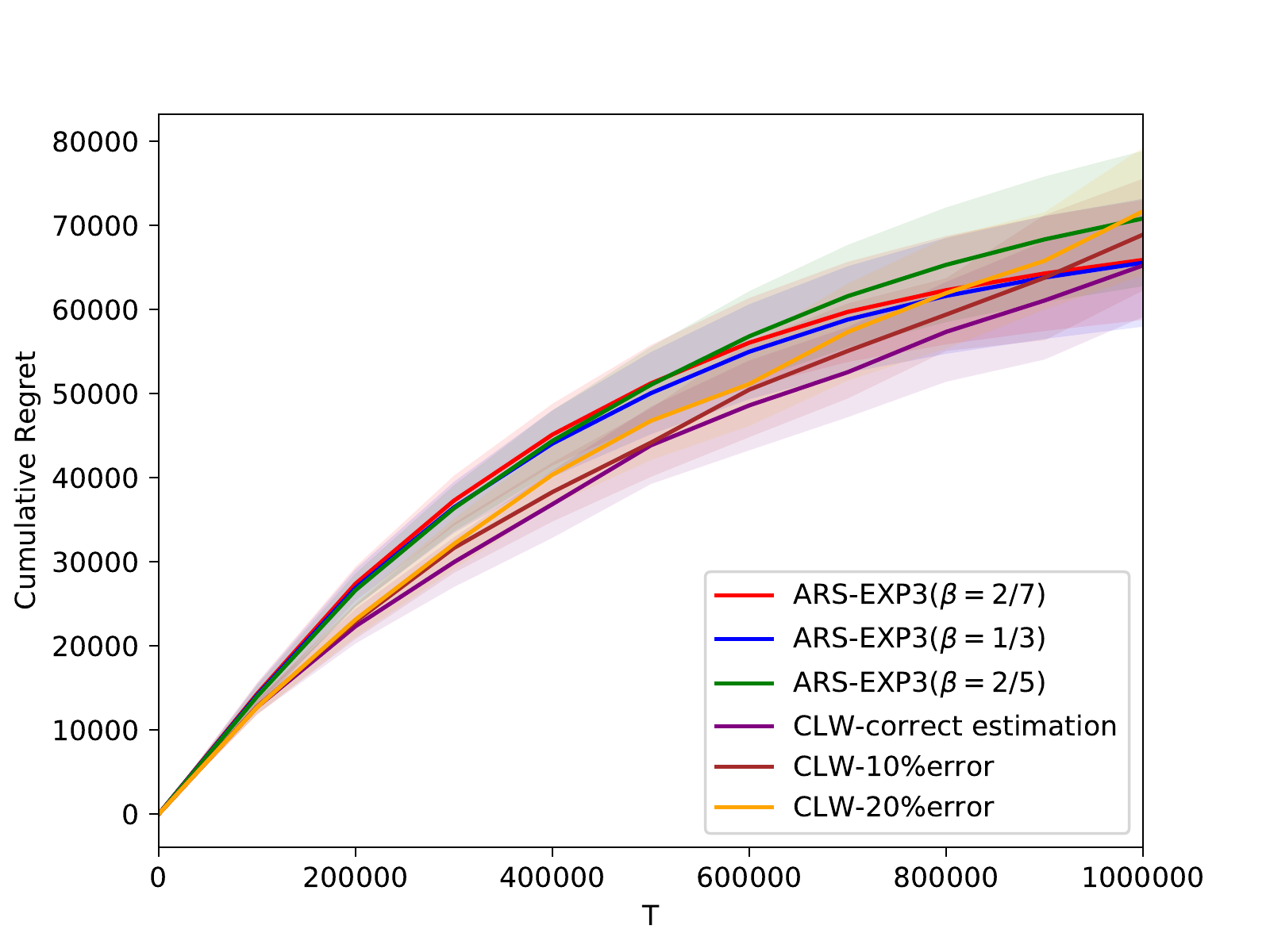}}
\subfigure[]{ \label{Figure_9} 
\includegraphics[width=1.34in]{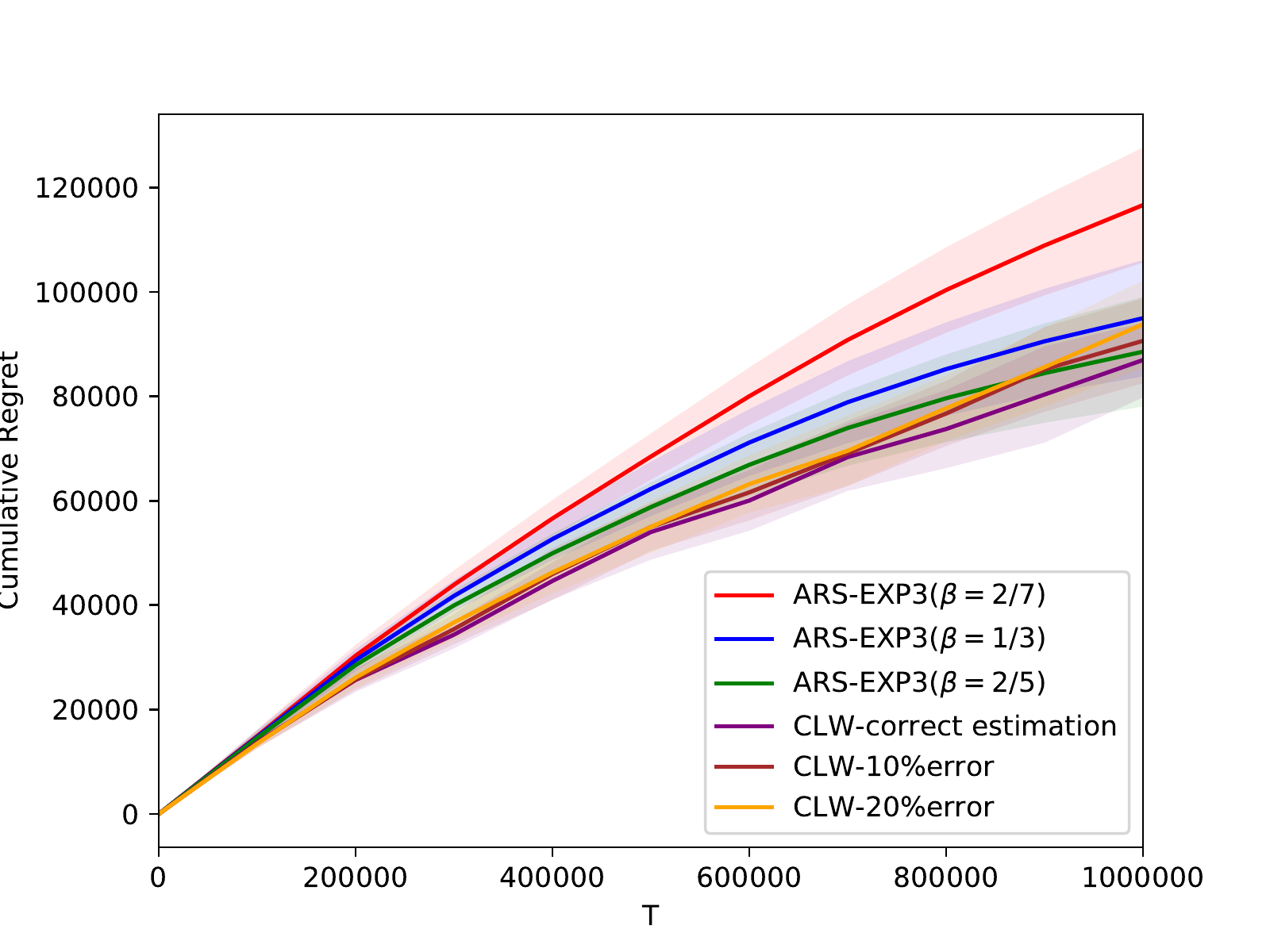}}
\caption{Experiments: Comparison between cumulative regrets of ARS-EXP3 and CLW
}
\end{figure}


In the experiments of Outbrain dataset, we consider non-oblivious delays with $z \le d = 10$ in  Figure \ref{Figure_6} and $z \le d = 20$ in  Figure \ref{Figure_7}. Here the adversary choose delay $z = d$ only if the chosen arm is the best one and it has been chosen for at least $3d$ times in succession, otherwise it set $z = 1$. From the experimental results, we can see that the CLW policy \cite{Bianchi2018Nonstochastic} suffers from a linear regret in this setting, while our ARS-EXP3 policy achieves a sub-linear regret. This accords to the theoretical analysis. 

Then, we use Coupon dataset to simulate the oblivious setting, in which the CLW policy also has theoretical guarantee. In Figure \ref{Figure_8}, we use $d=10$ and $z$ is chosen uniformly in $[5,10]$. In Figure \ref{Figure_9}, we set $d = 20$ and choose $z$ uniformly in $[10,20]$. From these results, we can see that CLW 
performs only slightly better than ARS-EXP3 even with the correct delay estimation $d$. 
When the estimation on $d$ has a $10\%$ error, CLW can perform worse than ARS-EXP3.

The experimental results show that in the oblivious setting, ARS-EXP3 policy is more robust, especially when there is no accurate information about the delay $d$. As for the non-oblivious delay case, ARS-EXP3 is the only existing efficient learning policy.   

\section{Conclusion} 

In this paper, we consider the MAB problem with composite and anonymous feedback, both the stochastic and adversarial settings. For the former case, we propose the ARS-UCB algorithm, and for the latter case, we design the ARS-EXP3 algorithm. 
These algorithms require \emph{zero} knowledge about the feedback delay. 
We establish theoretical regret upper bounds for the algorithms, and then use experiments to show that our algorithms outperform existing benchmarks.  

Our future research includes deriving a matching regret lower bound for the non-oblivious adversarial case. 
In \cite{Bianchi2018Nonstochastic}, the authors also provide a similar policy for bandit convex optimization (BCO) with delayed and anonymous feedback. 
How to adapt our framework and obtain tight regret upper bounds for BCO 
is  another  interesting future research problem.


%
%
%
%
%

%


\bibliography{sample-base}

\begin{thebibliography}{32}
\providecommand{\natexlab}[1]{#1}
\providecommand{\url}[1]{\texttt{#1}}
\providecommand{\urlprefix}{URL }
\expandafter\ifx\csname urlstyle\endcsname\relax
  \providecommand{\doi}[1]{doi:\discretionary{}{}{}#1}\else
  \providecommand{\doi}{doi:\discretionary{}{}{}\begingroup
  \urlstyle{rm}\Url}\fi

\bibitem[{Agarwal and Duchi(2011)}]{agarwal2011distributed}
Agarwal, A.; and Duchi, J.~C. 2011.
\newblock Distributed delayed stochastic optimization.
\newblock In \emph{Neural Information Processing Systems}, 873--881.

\bibitem[{Auer, Cesa-Bianchi, and Fischer(2002)}]{Auer2002Finite}
Auer, P.; Cesa-Bianchi, N.; and Fischer, P. 2002.
\newblock Finite-time analysis of the multiarmed bandit problem.
\newblock \emph{Machine learning} 47(2-3): 235--256.

\bibitem[{Auer et~al.(2002)Auer, Cesa-Bianchi, Freund, and
  Schapire}]{Auer2002The}
Auer, P.; Cesa-Bianchi, N.; Freund, Y.; and Schapire, R.~E. 2002.
\newblock The Non-Stochastic Multi-Armed Bandit Problem.
\newblock \emph{Siam Journal on Computing} 32(1): 48--77.

\bibitem[{Berry and Fristedt(1985)}]{Berry1985Bandit}
Berry, D.~A.; and Fristedt, B. 1985.
\newblock \emph{Bandit problems: sequential allocation of experiments
  (Monographs on statistics and applied probability)}.
\newblock Springer.

\bibitem[{Bistritz et~al.(2019)Bistritz, Zhou, Chen, Bambos, and
  Blanchet}]{bistritz2019online}
Bistritz, I.; Zhou, Z.; Chen, X.; Bambos, N.; and Blanchet, J. 2019.
\newblock Online exp3 learning in adversarial bandits with delayed feedback.
\newblock In \emph{Neural Information Processing Systems}, 11345--11354.

\bibitem[{Cesa-Bianchi, Dekel, and Shamir(2013)}]{cesa2013online}
Cesa-Bianchi, N.; Dekel, O.; and Shamir, O. 2013.
\newblock Online learning with switching costs and other adaptive adversaries.
\newblock In \emph{Neural Information Processing Systems}, 1160--1168.

\bibitem[{Cesa-Bianchi, Gentile, and Mansour(2018)}]{Bianchi2018Nonstochastic}
Cesa-Bianchi, N.; Gentile, C.; and Mansour, Y. 2018.
\newblock Nonstochastic bandits with composite anonymous feedback.
\newblock In \emph{Conference On Learning Theory}, 750--773.

\bibitem[{Chapelle, Manavoglu, and Rosales(2015)}]{chapelle2015simple}
Chapelle, O.; Manavoglu, E.; and Rosales, R. 2015.
\newblock Simple and scalable response prediction for display advertising.
\newblock \emph{ACM Transactions on Intelligent Systems and Technology (TIST)}
  5(4): 61.

\bibitem[{Chen et~al.(2018)Chen, Cai, Huang, and Lui}]{ck2018ijcai}
Chen, K.; Cai, K.; Huang, L.; and Lui, J.~C. 2018.
\newblock Beyond the click-through rate: web link selection with multi-level
  feedback.
\newblock In \emph{International Joint Conference on Artificial Intelligence},
  3308--3314.

\bibitem[{Chen, Wang, and Yuan(2013)}]{chen2013combinatorial}
Chen, W.; Wang, Y.; and Yuan, Y. 2013.
\newblock Combinatorial multi-armed bandit: General framework and applications.
\newblock In \emph{International Conference on Machine Learning}, 151--159.

\bibitem[{Dekel et~al.(2014)Dekel, Ding, Koren, and Peres}]{dekel2014bandits}
Dekel, O.; Ding, J.; Koren, T.; and Peres, Y. 2014.
\newblock Bandits with switching costs: $T^{2/3}$ regret.
\newblock In \emph{Proceedings of the forty-sixth annual ACM symposium on
  Theory of computing}, 459--467.

\bibitem[{Desautels, Krause, and Burdick(2014)}]{desautels2012parallelizing}
Desautels, T.; Krause, A.; and Burdick, J.~W. 2014.
\newblock Parallelizing exploration-exploitation tradeoffs in gaussian process
  bandit optimization.
\newblock \emph{Journal of Machine Learning Research} 15: 3873--3923.

\bibitem[{Garg and Akash(2019)}]{garg2019stochastic}
Garg, S.; and Akash, A.~K. 2019.
\newblock Stochastic bandits with delayed composite anonymous feedback.
\newblock \emph{arXiv preprint arXiv:1910.01161} .

\bibitem[{Gittins(1989)}]{gittins1989multi}
Gittins, J. 1989.
\newblock Multi-armed bandit allocation indices. Wiley-Interscience series in
  systems and optimization .

\bibitem[{Hirsch and Brownlee(2005)}]{Hirsch2005Should}
Hirsch, I.~B.; and Brownlee, M. 2005.
\newblock Should minimal blood glucose variability become the gold standard of
  glycemic control?
\newblock \emph{Journal of Diabetes and Its Complications} 19(3): 178--181.

\bibitem[{Jain and Jamieson(2018)}]{jain2018firing}
Jain, L.; and Jamieson, K. 2018.
\newblock Firing bandits: Optimizing crowdfunding.
\newblock In \emph{International Conference on Machine Learning}, 2211--2219.

\bibitem[{Joulani, Gyorgy, and Szepesv{\'a}ri(2013)}]{joulani2013online}
Joulani, P.; Gyorgy, A.; and Szepesv{\'a}ri, C. 2013.
\newblock Online learning under delayed feedback.
\newblock In \emph{International Conference on Machine Learning}, 1453--1461.

\bibitem[{Kaggle(2015)}]{Kaggle2}
Kaggle. 2015.
\newblock Coupon Purchase Prediction data.
\newblock \emph{https://www.kaggle.com/c/coupon-purchase-prediction} .

\bibitem[{Kaggle(2016)}]{Kaggle1}
Kaggle. 2016.
\newblock Outbrain Click Prediction data.
\newblock \emph{https://www.kaggle.com/c/outbrain-click-prediction} .

\bibitem[{Lai and Robbins(1985)}]{Lai1985Asymptotically}
Lai, T.~L.; and Robbins, H. 1985.
\newblock Asymptotically efficient adaptive allocation rules.
\newblock \emph{Advances in applied mathematics} 6(1): 4--22.

\bibitem[{Lu, P{\'a}l, and P{\'a}l(2010)}]{Lu2010}
Lu, T.; P{\'a}l, D.; and P{\'a}l, M. 2010.
\newblock Contextual multi-armed bandits.
\newblock In \emph{International conference on Artificial Intelligence and
  Statistics}, 485--492.

\bibitem[{Manegueu et~al.(2020)Manegueu, Vernade, Carpentier, and
  Valko}]{manegueu2020stochastic}
Manegueu, A.~G.; Vernade, C.; Carpentier, A.; and Valko, M. 2020.
\newblock Stochastic bandits with arm-dependent delays.
\newblock In \emph{International Conference on Machine Learning}.

\bibitem[{Neu et~al.(2010)Neu, Antos, Gy{\"o}rgy, and
  Szepesv{\'a}ri}]{neu2010online}
Neu, G.; Antos, A.; Gy{\"o}rgy, A.; and Szepesv{\'a}ri, C. 2010.
\newblock Online Markov decision processes under bandit feedback.
\newblock In \emph{Neural Information Processing Systems}, 1804--1812.

\bibitem[{Neu et~al.(2014)Neu, Gyorgy, Szepesvari, and Antos}]{neu2014online}
Neu, G.; Gyorgy, A.; Szepesvari, C.; and Antos, A. 2014.
\newblock Online Markov Decision Processes Under Bandit Feedback.
\newblock \emph{IEEE Transactions on Automatic Control} 59(3): 676--691.

\bibitem[{Pike-Burke et~al.(2018)Pike-Burke, Agrawal, Szepesvari, and
  Grunewalder}]{pikeburke2018bandits}
Pike-Burke, C.; Agrawal, S.; Szepesvari, C.; and Grunewalder, S. 2018.
\newblock Bandits with Delayed, Aggregated Anonymous Feedback.
\newblock In \emph{International Conference on Machine Learning}, 4105--4113.

\bibitem[{Slivkins(2014)}]{Sli2014}
Slivkins, A. 2014.
\newblock Contextual bandits with similarity information.
\newblock \emph{The Journal of Machine Learning Research} 15(1): 2533--2568.

\bibitem[{Sutton and Barto(1998)}]{Sutton1998Reinforcement}
Sutton, R.~S.; and Barto, A.~G. 1998.
\newblock \emph{Reinforcement learning: An introduction}, volume~1.
\newblock MIT press Cambridge.

\bibitem[{Thune, Cesa-Bianchi, and Seldin(2019)}]{thune2019nonstochastic}
Thune, T.~S.; Cesa-Bianchi, N.; and Seldin, Y. 2019.
\newblock Nonstochastic Multi-armed Bandits with Unrestricted Delays.
\newblock In \emph{Neural Information Processing Systems}.

\bibitem[{Vernade, Capp{\'e}, and Perchet(2017)}]{Vernade2010Stochastic}
Vernade, C.; Capp{\'e}, O.; and Perchet, V. 2017.
\newblock Stochastic Bandit Models for Delayed Conversions.
\newblock In \emph{Conference on Uncertainty in Artificial Intelligence}.

\bibitem[{Wang and Huang(2018)}]{wang2018multi}
Wang, S.; and Huang, L. 2018.
\newblock Multi-armed bandits with compensation.
\newblock In \emph{Neural Information Processing Systems}, 5114--5122.

\bibitem[{Weinberger and Ordentlich(2002)}]{weinberger2002on}
Weinberger, M.~J.; and Ordentlich, E. 2002.
\newblock On delayed prediction of individual sequences.
\newblock \emph{international symposium on information theory} 48(7):
  1959--1976.

\bibitem[{Zhou, Xu, and Blanchet(2019)}]{zhou2019learning}
Zhou, Z.; Xu, R.; and Blanchet, J. 2019.
\newblock Learning in generalized linear contextual bandits with stochastic
  delays.
\newblock In \emph{Neural Information Processing Systems}, 5198--5209.

\end{thebibliography}

\onecolumn

\clearpage

\appendix

\section*{Appendix}

The proofs of all lemmas are presented in the end of the sections they belong to. 

\section{Proof of Theorem \ref{Theorem_UCB}}\label{Section_proof_1}

Denote $rad_i(t) = \sqrt{\alpha\log t\over N_i(t)}$. Then, for each sub-optimal arm $i \ne 1$, if it is chosen at time $t$, we must have either $u_i(t) = 1$ or $\hat{s}_i(t) + rad_i(t) \ge \hat{s}_1(t) + rad_1(t)$. 

Since $s_1 = s_i + \Delta_i$, $\hat{s}_i(t) + rad_i(t) \ge \hat{s}_1(t) + rad_1(t)$ implies that: 
\begin{equation*}
    \hat{s}_i(t) + s_1 + 2rad_i(t) \ge s_i+rad_i(t)+\hat{s}_1(t) + rad_1(t) + \Delta_i.
\end{equation*}

Now we define some useful events. Note that if we choose to pull arm $i$ at time $t$, one of the following events must happen: 
\begin{eqnarray*}
  \mathcal{A}_i(t) &=& \{\hat{s}_i(t) \ge s_i+rad_i(t)\},\\
  \mathcal{B}(t) &=& \{s_1 \ge \hat{s}_1(t)+rad_1(t)\},\\
  \mathcal{C}_i(t) &=& \{2rad_i(t) \ge \Delta_i\},\\
  \mathcal{D}_i(t) &=& \{\hat{s}_i(t) + rad_i(t) \ge 1\}.
\end{eqnarray*}

Let $\mathcal{U}$ be the set that includes all the last time steps of all the rounds, and note that $\mathcal{U}$ is random. Also let $z(t)$ be the size of the next round after time $t$, 
then we have
\begin{eqnarray}
	\nonumber Reg_i(T) &=& \E_{\mathcal{U}}\left[ \sum_{t\in \mathcal{U}}\E[\I[a(t+1) = i]]\Delta_i z(t)\right]\\
	\nonumber &\le& \E_{\mathcal{U}}\left[\sum_{t\in \mathcal{U}} \E[\I[\mathcal{A}_i(t) \cup \mathcal{B}(t) \cup \mathcal{C}_i(t) \cup \mathcal{D}_i(t)]]\Delta_i z(t)\right]\\
	\nonumber &=& \E_{\mathcal{U}}\left[ \sum_{t\in \mathcal{U}}\Pr[\mathcal{A}_i(t) \cup \mathcal{B}(t) \cup \mathcal{C}_i(t) \cup \mathcal{D}_i(t)]\Delta_i z(t)\right],
\end{eqnarray}
where $Reg_i(T)$ denotes the expected cumulative regret caused by pulling arm $i$ up to time step $T$, i.e., $Reg_i(T) = \E[N_i(T)]\Delta_i$.


First consider events $\mathcal{A}_i(t)$ and $\mathcal{B}(t)$. Since our observations have bias (due to reward spread), it is hard to analyze the inequalities that contains $rad_i(t)$. Instead, we construct a hidden confidence radius as
\begin{eqnarray*}
rad'_i(t) \triangleq   \sqrt{4\log t \over N_i(t)} +  d_1{K_i(t)\over N_i(t)} + \sqrt{12d_2K_i(t)\log t \over N_i^2(t)}, 
\end{eqnarray*} 
and the corresponding hidden confidence bounds are defined as:
\begin{eqnarray*}
\ell_i(t) &\triangleq& \min\{\hat{s}_i(t) - rad'_i(t),0\},\\
v_i(t) &\triangleq& \max\{\hat{s}_i(t) + rad'_i(t),1\}.
\end{eqnarray*}

The next lemma shows that $\ell_i(t), v_i(t)$ are actual confidence bounds for $s_i$.

\begin{lemma}\label{Lemma_1main}
    For any $t > 0$ and $1 \le i \le N$, with probability at least $1 - {1\over t^3}$, we have that $s_i \le v_i(t)$. Similarly, with probability at least $1 - {1\over t^3}$, $s_i \ge \ell_i(t)$.
\end{lemma}

%
Now we define another two events based on $rad'_i(t)$:
\begin{eqnarray*}
  \mathcal{E}_i(t) &=& \{|\hat{s}_i(t) - s_i | \ge rad'_i(t)\}, \\
  \mathcal{F}_i(t) &=& \{rad_i(t) \ge rad'_i(t)\}.
\end{eqnarray*}

Since $\hat{s}_i(t) \ge s_i + rad_i(t)$ implies either $\{rad_i(t) < rad'_i(t)\}$ or $\{rad_i(t) \ge rad'_i(t)\}\cap \{ |\hat{s}_i(t) - s_i | \ge rad'_i(t)\}$, we know that
\begin{eqnarray}
\mathcal{A}_i(t) \subseteq \neg \mathcal{F}_i(t) \cup \mathcal{E}_i(t).\label{eq:A-eq}
\end{eqnarray}
Similarly,we also have that
\begin{eqnarray}
\mathcal{B}(t) \subseteq \neg \mathcal{F}_1(t) \cup \mathcal{E}_1(t).\label{eq:B-eq}
\end{eqnarray}

Now we consider the event $\mathcal{D}_i(t)$.
When the three events $\mathcal{D}_i(t)$, $\neg \mathcal{E}_i(t)$ and $ \mathcal{F}_i(t)$ happens, we have that
\begin{eqnarray}
\nonumber s_i + 2rad_i(t) &= & (s_i + rad_i(t) )+ rad_i(t)\\
\label{eq:1000}&\ge&s_i + rad'_i(t) + rad_i(t)\\
\label{eq:1001}&\ge&\hat{s}_i(t) + rad_i(t) \\
\label{eq:1002}&\ge& 1\\
\label{eq:1003}&\ge& s_i + \Delta_i,
\end{eqnarray}
where Eq. \eqref{eq:1000} comes from the fact that $\mathcal{F}_i(t)$ happens, Eq. \eqref{eq:1001} comes from the fact that $\neg \mathcal{E}_i(t)$ happens, Eq. \eqref{eq:1002} comes from the fact that $\mathcal{D}_i(t)$ happens, and Eq. \eqref{eq:1003} is because that $s_i + \Delta_i = s_1 \le 1$ as we have assumed in our model setting. 

This means that $2rad_i(t) \ge \Delta_i$, i.e., the event $\mathcal{C}_i(t)$ must happen if we have $\mathcal{D}_i(t) \cap \neg \mathcal{E}_i(t) \cap \mathcal{F}_i(t)$. Therefore, $\mathcal{D}_i(t) \cap \neg \mathcal{E}_i(t) \cap \mathcal{F}_i(t) \subseteq \mathcal{C}_i(t)$, which implies that 
\begin{eqnarray}
\mathcal{D}_i(t) \subseteq \mathcal{E}_i(t) \cup  \neg \mathcal{F}_i(t) \cup \mathcal{C}_i(t).\label{eq:D-eq}
\end{eqnarray}

From equations \eqref{eq:A-eq}, \eqref{eq:B-eq} and \eqref{eq:D-eq},  we know that 
\begin{equation*}
\Pr[\mathcal{A}_i(t) \cup \mathcal{B}(t) \cup \mathcal{C}_i(t) \cup \mathcal{D}_i(t)] \le \Pr[\neg \mathcal{F}_i(t)\cup \mathcal{E}_i(t)\cup \neg \mathcal{F}_1(t)\cup \mathcal{E}_1(t)\cup \mathcal{C}_i(t)].
\end{equation*}

Therefore, 
\begin{eqnarray}
	\nonumber Reg_i(T) &\le& \E_{\mathcal{U}}\left[ \sum_{t\in \mathcal{U}}\Pr[\neg \mathcal{F}_i(t)\cup \mathcal{E}_i(t)\cup \neg \mathcal{F}_1(t)\cup \mathcal{E}_1(t)\cup \mathcal{C}_i(t)]\Delta_i z(t)\right]\\
\nonumber &\le& \E_{\mathcal{U}}\left[ \sum_{t\in \mathcal{U}}\Pr[ \mathcal{E}_i(t)\cup \mathcal{E}_1(t)\cup \mathcal{C}_i(t)]\Delta_i z(t)\right] + \E_{\mathcal{U}}\left[ \sum_{t\in \mathcal{U}}\Pr[\neg \mathcal{F}_i(t)\cup \neg \mathcal{F}_1(t)]\Delta_i z(t)\right] 
\end{eqnarray}


The next lemma (Lemma \ref{Lemma_2main}) shows that there exists a constant $T^*$ (which does not depend on $T$) such that $\Pr[\neg \mathcal{F}_i(t)\cup \neg \mathcal{F}_1(t)] = 0$ for any $t \ge T^*$.

\begin{lemma}\label{Lemma_2main}
    If $f$ is an increasing function, then there exists $T^* = c_f(d_1, d_2, N,\alpha)$ such that for any $i$ and any $t \ge T^*$,  $rad_i(t) \ge rad'_i(t)$.
\end{lemma}

Then we can divide the game into two phases: the phase that lasts until the round contains $T^*$ stops, and the phase that contains the remaining time steps. Let $\ell(T^*)$ be the size of the round which contains time step $T^*$.

The regret in the first phase is upper bounded by $Reg^{(1)}(T) \le T^* + \E[\ell(T^*)]$. Note that the property ii) of $f$ in Theorem \ref{Theorem_UCB} states that $f(k+1) \le F(k)$ for any $k > k_0$, then 
\begin{eqnarray*}
\E[\ell(T^*)] &\le& \E[\ell(T^*)| \ell(T^*) \le f(k_0)] + \E[\ell(T^*)| \ell(T^*) > f(k_0)] \\
&\le& f(k_0) + T^*.
\end{eqnarray*}

Now we consider the second phase, and denote $Reg_i^{(2)}(T)$ as the regret comes from pulling arm $i$ in the second phase. Since in the second phase $\neg \mathcal{F}_i(t)$ can not happens (Lemma \ref{Lemma_2main}), we have that:
\begin{eqnarray}
\nonumber  Reg_i^{(2)}(T) &\le& \E_{\mathcal{U}}\left[ \sum_{t\in \mathcal{U}, t > T^*}\Pr[\mathcal{E}_i(t)\cup \mathcal{E}_1(t)\cup \mathcal{C}_i(t)]\Delta_i z(t)\right]\\
\nonumber &\le& \E_{\mathcal{U}}\left[ \sum_{t\in \mathcal{U}, t > T^*}  \I[z(t) > f(k_0)]z(t)\Delta_i  \Pr[\mathcal{E}_i(t)\cup \mathcal{E}_1(t)] \right] \\
\nonumber &&+ \E_{\mathcal{U}}\left[ \sum_{t\in \mathcal{U}, t > T^*}  \I[z(t) > f(k_0)]z(t)\Delta_i  \Pr[\mathcal{C}_i(t)] \right] \\
\nonumber &&+ \E_{\mathcal{U}}\left[ \sum_{t\in \mathcal{U}, t > T^*}  \I[z(t) \le f(k_0)]z(t)]  \right]\\
\nonumber &\le&\E_{\mathcal{U}}\left[ \sum_{t\in \mathcal{U}, t > T^*}  \I[z(t) > f(k_0)]z(t)\Delta_i  \Pr[\mathcal{E}_i(t)\cup \mathcal{E}_1(t)] \right] \\
\nonumber &&+ \E_{\mathcal{U}}\left[ \sum_{t\in \mathcal{U}, t > T^*}  \I[z(t) > f(k_0)]z(t)\Delta_i  \Pr[\mathcal{C}_i(t)] \right] + \sum_{k=1}^{k_0} f(k)  \\
\nonumber &\le& \E_{\mathcal{U}}\left[ \sum_{t\in \mathcal{U}, t > T^*}  \I[z(t) > f(k_0)]z(t)\Delta_i  \Pr[\mathcal{E}_i(t)\cup \mathcal{E}_1(t)] \right] \\
\label{Eq_12221}&&+ \E_{\mathcal{U}}\left[ \sum_{t\in \mathcal{U}, t > T^*}  \I[z(t) > f(k_0)]z(t)\Delta_i  \Pr[\mathcal{C}_i(t)] \right] + F(k_0) . 
\end{eqnarray}

Consider the first term in Eq. \eqref{Eq_12221}, we have that
\begin{eqnarray}
\E_{\mathcal{U}}\left[ \sum_{t\in \mathcal{U}, t > T^*}  \I[z(t) > f(k_0)]z(t)\Delta_i  \Pr[\mathcal{E}_i(t)\cup \mathcal{E}_1(t)] \right]
\label{eq:1010}&\le &\E_{\mathcal{U}}\left[ \sum_{t\in \mathcal{U}, t > T^*}  t\Delta_i  \Pr[\mathcal{E}_i(t)\cup \mathcal{E}_1(t)] \right]\\
\label{eq:1011}&\le& \E_{\mathcal{U}}\left[ \sum_{t\in \mathcal{U}, t > T^*}  t\Delta_i  {4\over t^3}\right]\\
\nonumber &\le& \E_{\mathcal{U}}\left[ \sum_{t\in \mathcal{U}, t > T^*}  {4\over t^2}\right]\\
\nonumber &\le& 8,
\end{eqnarray}
where Eq. \eqref{eq:1010} comes from the fact that $\I[z(t) > f(k_0)]z(t) \le t$ (note that $f(k+1) \le F(k)$ for any $k > k_0$) and Eq. \eqref{eq:1011} comes from Lemma \ref{Lemma_1main}.

Then we consider the second term in Eq. \eqref{Eq_12221}. Since $\mathcal{C}_i(t)$ only happens when $N_i(t) \le {4\alpha\log T \over \Delta_i^2}$, we have that
\begin{eqnarray}
\nonumber \E_{\mathcal{U}}\left[ \sum_{t\in \mathcal{U}, t > T^*}  \I[z(t) > f(k_0)]z(t)\Delta_i  \Pr[\mathcal{C}_i(t)] \right]
&\le& \Delta_i \left({4\alpha\log T \over \Delta_i^2} + \E\left[\I[\ell_i > f(k_0)]\ell_i\right]\right),
\end{eqnarray}
where $\ell_i$ represents the last round size that $\mathcal{C}_i(t)$ happens. 

Since $f(k+1) \le F(k)$ for any $k > k_0$, we have $\E\left[\I[\ell_i > f(k_0)]\ell_i\right] \le {4\alpha\log T \over \Delta_i^2}$ as well, therefore:
\begin{equation*}
\E\left[\sum_{i=2}^N \sum_{t\in \mathcal{U}, t \ge T^*}^T \I[z(t) > f(k_0)]z(t)\Delta_i\I[\mathcal{C}_i(t)]\right] \le {8\alpha\log T \over \Delta_i}.
\end{equation*}

Adding all these three terms together, and summing over all the sub-optimal arms, we have that:
\begin{eqnarray*}
Reg(T) &\le& Reg^{(1)}(T) + \sum_{i=2}^N Reg^{(2)}_i(T) \\
&\le& 2T^* + f(k_0) + 8(N-1) + (N-1)F(k_0) + \sum_{i=2}^N {8\alpha\log T \over \Delta_i}.
\end{eqnarray*}

Setting $c_f^*(d_1,d_2,N,\alpha) = 2T^* + 8N + NF(k_0)$, then we know that the total regret is upper bounded by:
\begin{equation*}
    Reg(T) \le \sum_{i=2}^N {8\alpha \log T \over \Delta_i} + c_f^*(d_1,d_2,N,\alpha).
\end{equation*}

\subsection{Proof of Lemma \ref{Lemma_1main}}\label{sub_section_sto_proof2}

Define $M'_i(k) \triangleq \sum_{t = 1}^\infty \I[K_i(t) = k, a(t) = i]Y(t)$, i.e., the observed cumulative reward of arm $i$ in its $k$-th round. Similarly, define $L'_i(k) \triangleq \sum_{t = 1}^\infty \I[K_i(t) = k, a(t) = i]||\bm{r}_{a(t)}(t)||_1$, i.e., the real cumulative reward of arm $i$ in its $k$-th round. Then at the end of each round, we always have $M_i(t) = \sum_{k=1}^{K_i(t)}M'_i(k)$ and $L_i(t) = \sum_{k=1}^{K_i(t)}L'_i(k)$.

Note that in each round we only pull one arm $i$, then by equation \eqref{eq_111}, we always have 
\begin{eqnarray}
\label{eq_222} M'_i(k) = L'_i(k) + \sum_{t \le t_1-1} \sum_{\tau = t_1-t}^{\infty} r_{a(t),\tau}(t) - \sum_{t \le t_2}\sum_{\tau = t_2-t+1}^{\infty} r_{a(t),\tau}(t),
\end{eqnarray} 
where $t_1$ is the start of that round, and $t_2$ is the end of that round. 
Taking expectation on the both sides in equation \eqref{eq_222}, and recall that $d_1 \triangleq \sum_{d' = 1}^d \max_i \E[\sum_{\tau = d'}^{d} r_{i,\tau}]$, we obtain 
\begin{eqnarray*}
\E[M'_i(k)] &=& \E[L_i'(k)] + \E\left[\sum_{t \le t_1-1} \sum_{\tau = t_1-t}^{\infty} r_{a(t),\tau}(t) - \sum_{t \le t_2}\sum_{\tau = t_2-t+1}^{\infty} r_{a(t),\tau}(t)\right]\\
&=& f(k)s_i + \E\left[\sum_{t \le t_1-1} \sum_{\tau = t_1-t}^{\infty} r_{a(t),\tau}(t) - \sum_{t \le t_2}\sum_{\tau = t_2-t+1}^{\infty} r_{a(t),\tau}(t)\right]\\
&\ge& f(k)s_i - d_1.
\end{eqnarray*}

%
Summing over $k=1$ to $K_i(t)$, we have that
\begin{equation}\label{eq_223}\E[M_i(t)] \ge N_i(t)s_i - K_i(t)d_1.\end{equation}

Similarly, taking variance on the both sides in equation \eqref{eq_222}, and recall that $d_2 \triangleq \sum_{d' = 1}^d \max_i \Var[\sum_{\tau = d'}^{d} r_{i,\tau}]$, we have  
\begin{eqnarray*}
\Var[M'_i(k)] &=& \Var[L_i'(k)] + \Var\left[\sum_{t \le t_1-1} \sum_{\tau = t_1-t}^{\infty} r_{a(t),\tau}(t) - \sum_{t \le t_2}\sum_{\tau = t_2-t+1}^{\infty} r_{a(t),\tau}(t)\right]  \\
&\le& {f(k) \over 4} + \Var\left[\sum_{t \le t_1-1} \sum_{\tau = t_1-t}^{\infty} r_{a(t),\tau}(t) - \sum_{t \le t_2}\sum_{\tau = t_2-t+1}^{\infty} r_{a(t),\tau}(t)\right]  \\
&\le& {f(k) \over 4} + 2d_2,
\end{eqnarray*}

%
Summing over $k=1$ to $K_i(t)$, we have that \begin{equation}\label{eq_224}\Var[M_i(t)] \le {N_i(t) \over 4} + 2K_i(t)d_2.\end{equation}

By inequalities \eqref{eq_223} and \eqref{eq_224}, using Bernstein's inequality, for any $\Delta \in [0,1]$, we have that:
\begin{eqnarray*}
\Pr\left[{M_i(t) \over N_i(t)} \le s_i - {K_i(t) \over N_i(t)}d_1 - \Delta\right]
&=&\Pr[M_i(t) \le N_i(t)s_i - K_i(t)d_1 - N_i(t)\Delta] \\
&\le& \exp\left(-{3N_i(t)^2\Delta^2 \over 6\Var[M_i(t)] + 2N_i(t)\Delta}\right)\\
&\le& \exp\left(-{3N_i(t)^2\Delta^2 \over 4N_i(t) + 12K_i(t)d_2}\right).
\end{eqnarray*}

Then, we can set $\Delta = \min\{1, \sqrt{4\log t \over N_i(t)} + \sqrt{12d_2K_i(t)\log t\over N_i(t)^2}\}$, which implies that $\Pr[\hat{s}_i(t) < s_i - rad'_i(t)] \le {1\over t^3}$.

Similarly, we also have $\Pr[\hat{s}_i(t) > s_i + rad'_i(t)]\le {1\over t^3}$, which completes the proof of this lemma.

\subsection{Proof of Lemma \ref{Lemma_2main}}\label{sub_section_sto_proof3}

When $f$ is an increasing function, we must have $F(K) \ge {K(K+1)\over 2}$. Therefore, $\lim_{K\to\infty}{K\over F(K)} = 0$ and ${K\over \sqrt{F(K)}} \le 2$.

Note that it is sufficient to prove that there exists some $T^*$ such that $\forall t \ge T^*$, the following two inequalities always hold: 
\begin{eqnarray}\label{eq_888}\left({\sqrt{\alpha} - 2 \over 2}\right) \sqrt{\log t\over N_i(t)} &\ge& d_1 {K_i(t) \over N_i(t)},\\
\label{eq_889}\left({\sqrt{\alpha} - 2 \over 2}\right) \sqrt{\log t\over N_i(t)} &\ge& \sqrt{12d_2K_i(t)\log t \over N_i(t)^2}.\end{eqnarray}

Eq.  \eqref{eq_888} is the same as 
\begin{equation*}
\left({\sqrt{\alpha} - 2 \over 2}\right) \sqrt{\log t} \ge d_1 {K_i(t) \over \sqrt{N_i(t)}} = d_1 {K_i(t) \over \sqrt{F(K_i(t))}}.
\end{equation*} 

Since ${k\over \sqrt{F(k)}} \le 2$, there exists $t_1 =\exp\left(\bigg({4d_1 \over \sqrt{\alpha }-2}\bigg)^2\right)$ such that for any $t \ge t_1$, we always have 
\begin{equation*}
\left({\sqrt{\alpha} - 2 \over 2}\right) \sqrt{\log t} \ge \left({\sqrt{\alpha} - 2 \over 2}\right) \sqrt{\log t_1} = 2d_1 \ge d_1{K_i(t) \over \sqrt{N_i(t)}}.
\end{equation*}

Now we consider Eq. \eqref{eq_889}, which is the same as
\begin{equation}\label{eq_890}
\sqrt{12d_2K_i(t)\over N_i(t)} = \sqrt{12d_2K_i(t)\over F(K_i(t))} \le {\sqrt{\alpha} - 2 \over 2}.
\end{equation}

Since ${K_i(t) \over F(K_i(t))}$ converges to $0$ as $K_i(t)$ (or equivalently $N_i(t)$) goes to infinity, there must exist some $N(d_2)$ such that for any $N_i(t)> N(d_2)$, Eq. \eqref{eq_890} always holds. 
Therefore, we only need to find some $t_2$ such that under Algorithm \ref{Algorithm_UCB}, each arm is pulled for at least $N(d_2)$ times. Note that for $t > \exp({N(d_2)\over \alpha})$,  any arm with $N_i(t) \le N(d_2)$ must have $u_i(t) = 1$. 
This implies that either arm $i$ will be pulled, or arm $j$ with $N_j(t) \le N_i(t)$ (so that $u_i(t)=u_j(t)=1$) will be pulled. Thus, after $t_2 = \exp({N(d_2)\over \alpha}) + N (F(k_0)+ 2N(d_2))$ time slots, every arm must be pulled for at least $N(d_2)$ times, i.e., for any $t \ge t_2$, Eq. \eqref{eq_889} always holds. 

Let $T^*= \max \{t_1, t_2\} = c_f(d_1, d_2, N, \alpha) $, we know that $rad_i(t) \ge rad'_i(t)$ for any $i$ and $t \ge T^*$.

\section{Proof of Theorem \ref{Theorem_EXP3}}\label{Section_proof_2}

Firstly, we define $L_i(t) = \sum_{\tau = 1}^t ||r_i(\tau)||_1$ as the actual reward gained from arm $i$ until time $t$, and $G(K) = \sum_{k=1}^K g(k)$ to be the total time steps until the $k$-th round finishes. 
We also set $\hat{z}_i(k) = \I[a(k) = i]{Z'(k)\over p_{i}(k)}$ to simplify  writing. 
Then similar to the analysis of EXP3 \cite{Auer2002The}, we have the following three inequalities:
\begin{equation}
\label{eq_449}{\gamma \over N} {\hat{z}_i(k) \over g(K)} \le p_i(k){\hat{z}_i(k) \over g(K)} = \I[a(k) = i]{Z'(k)\over g(K)} \le 1.
\end{equation}
\begin{equation}
\label{eq_450}\sum_{i=1}^N p_i(k){\hat{z}_i(k) \over g(K)} = {Z'(k)\over g(K)}.
\end{equation}
\begin{equation}
\label{eq_451}\sum_{i=1}^N p_i(k)\left({\hat{z}_i(k) \over g(K)}\right)^2 = {Z'(k) \over g(K)^2}\I[a(k) = i]\hat{z}_i(k) \le {1 \over g(K)}\sum_{i=1}^N \hat{z}_i(k) .
\end{equation}

We set $w_i(k)$ to be the the value of $w_i$ at the beginning of round $k$, $e_i(k)$ to be the value of $e_i$ in round $k$. $E(k)$ is the sum of all $e_i(k)$, i.e., $E(k) = \sum_{i=1}^N e_i(k)$.


Then we can write the following inequality:
\begin{eqnarray}
\nonumber {E(k+1) \over E(k)}
&=& \sum_{i=1}^N {e_i(k+1) \over E(k)}\\
\nonumber &=&\sum_{i=1}^N {1 \over E(k)} \exp\left({w_i(k+1)\over g(K)} \right)\\ 
\nonumber &=&\sum_{i=1}^N {1 \over E(k)} \exp\left({w_i(k) + {\gamma \over N}\hat{z}_i(k)\over g(K)} \right)\\ 
\nonumber &=&\sum_{i=1}^N {e_i(k) \over E(k)} \exp\left( {{\gamma \over N}\hat{z}_i(k)\over g(K)}\right)\\ 
\label{eq_444}&=& \sum_{i=1}^N {p_i(k) - {\gamma \over N} \over 1-\gamma} \exp\left({\gamma \over N} {\hat{z}_i(k) \over g(K)}\right)\\
\label{eq_445}&\le& \sum_{i=1}^N {p_i(k) - {\gamma \over N} \over 1-\gamma} \left(1 + {\gamma \over N} {\hat{z}_i(k) \over g(K)} + (e-2) ({\gamma \over N} {\hat{z}_i(k) \over g(K)})^2 \right)\\
\nonumber&\le& 1 + { {\gamma \over N} \over 1-\gamma} \sum_{i=1}^N p_i(k) {\hat{z}_i(k) \over g(K)} + {(e-2) ({\gamma \over N})^2 \over 1-\gamma} \sum_{i=1}^N p_i(k) \left({\hat{z}_i(k) \over g(K)}\right)^2\\
\label{eq_446}&\le& 1+ { {\gamma \over N} \over 1-\gamma}{Z'(k)\over g(K)} + {(e-2) ({\gamma \over N})^2 \over 1-\gamma} {1\over g(K)}\sum_{i=1}^N \hat{z}_i(k),
\end{eqnarray}
where Eq. \eqref{eq_444} comes from the fact that $p_i(k) = (1-\gamma){e_i(k) \over E_i(k)} + {\gamma \over N}$, Eq. \eqref{eq_445} is because of the fact that $\exp(x) \le 1 + x + (e-2)x^2$ for any $0\le x \le 1$ and Eq. \eqref{eq_449}, and Eq. \eqref{eq_446} comes from Eq. \eqref{eq_450} and Eq. \eqref{eq_451}.

Therefore,
\begin{equation*}
	\log {E(K) \over E(0)} \le {{\gamma \over N} \over 1-\gamma} \sum_{k=1}^K {Z'(k)\over g(K)} + {(e-2) ({\gamma \over N})^2 \over 1-\gamma} {1\over g(K)} \sum_{k=1}^K\sum_{i=1}^N \hat{z}_i(k).
\end{equation*}

On the other hand, for any arm $i$, we have that:
\begin{equation*}
	\log{E(K) \over E(0)} \ge \log {e_i(K) \over E(0)} = {\gamma \over N} \sum_{k=1}^K {\hat{z}_i(k)\over g(K)} - \log N.
\end{equation*}

Thus, 
\begin{equation}
\sum_{k=1}^K Z'(k) \ge (1-\gamma) \sum_{k=1}^K \hat{z}_i(k) - g(K){N\log N \over \gamma} - (e-2) {\gamma \over N} \sum_{k=1}^K\sum_{i=1}^N \hat{z}_i(k).\label{eq:z-ineq-0}
\end{equation}

Notice that by definitions of $Z'(k)$ (in in Algorithm \ref{Algorithm_EXP3}) and $\hat{z}_i(k)$, we have that:
\begin{eqnarray*}
\E[\hat{z}_i(k)|a(1),\cdots a(G(k)-1)]
&=& \min\{ g(k), \sum_{t=G(k-1) + 1}^{G(k)} ||r_i(t)||_1 - \sum_{t = G(k) - d}^{G(k)} \sum_{\tau = G(k) - t + 1}^d r_{a(t),\tau}(t) \\
&& + \sum_{t = G(k-1) - d}^{G(k-1)} \sum_{\tau = G(k-1) - t + 1}^d r_{a(t),\tau}(t)\}.
\end{eqnarray*}

This implies that
\begin{eqnarray*}
\sum_{t=G(k-1) + 1}^{G(k)} ||r_i(t)||_1 - d &\le& \E[\hat{z}_i(k)|a(1),\cdots a(G(k)-1)]
 \le \sum_{t=G(k-1) + 1}^{G(k)} ||r_i(t)||_1 + d.
\end{eqnarray*}

Summing the above equation over rounds $1, ..., k$, we obtain  
\begin{equation*}
L_i(G(k)) - k d \le \sum_{j=1}^k \E[\hat{z}_i(j)|a(1),\cdots a(G(j)-1)] \le L_i(G(k)) + kd .
\end{equation*}

Then, using Eq. \eqref{eq:z-ineq-0}, we have:  
\begin{eqnarray*}
\E\left[\sum_{k=1}^K Z'(k)\right] &\ge& (1-\gamma) \left(L_i(G(K)) - Kd\right) - g(K){N\log N \over \gamma}  - (e-2){\gamma \over N}\sum_{i=1}^N \left(L_i(G(K)) + Kd\right) .
\end{eqnarray*}

Let $i$ be the arm such that $L_i(G(K)) \ge L_j(G(K))$ for any $j\ne i$. Then, 
\begin{eqnarray}
\nonumber L_i(G(K)) - \E\left[\sum_{k=1}^K Z(k)\right]
&\le& L_i(G(K)) - \E\left[\sum_{k=1}^K Z'(k)\right] \\
&\le& (e-1)\gamma L_i(G(K)) + g(K){N\log N \over \gamma} + ((e-2)\gamma + 1-\gamma)Kd.\label{EQ_99910}
\end{eqnarray}

Note that the left-hand-side of Eq. \eqref{EQ_99910} is the cumulative  regret in the first $K$ rounds (or the first $G(K)$ time steps), and if we choose $\gamma = \min \{1, \sqrt{(N\log N) \over (e-1)K}\}$,  the right-hand-side of Eq. \eqref{EQ_99910} is upper bounded by $3g(K)\sqrt{N\log N K} + 2dK$. Thus, the cumulative regret in the first $G(K)$ time steps has an upper bound of $3g(K)\sqrt{N\log N K} + 2dK$. 

In the remaining time steps, the total regret is less than $g(K+1)$. 
Therefore, the total cumulative regret is upper bounded by $3g(K)\sqrt{N\log N K} + 2dK + g(K+1)$. 

Now if we choose $g(k)$ to be $k^\beta$ with $\beta\in[0, 1]$, then $G(k) = {1\over \beta + 1} k^{\beta + 1}$ and  $K = ((\beta + 1) T)^{1\over \beta + 1}$, which means that the total cumulative regret is upper bounded by:
\begin{equation*}
Reg(T) \le 3g(K)\sqrt{N\log N K} + 2dK + g(K+1) = O((N\log N)^{1\over 2} T^{2\beta + 1\over 2\beta + 2}+ dT^{1\over \beta + 1}). 
\end{equation*}

\section{Oblivious Adversarial MAB with Composite and Anonymous  Rewards}\label{Sub-CLW-2}



We also consider the oblivious adversarial MAB model with composite and anonymous rewards \cite{Bianchi2018Nonstochastic}, and present an ARS-CLW policy to achieve a regret upper bound that can be arbitrarily close to $O(\sqrt{NT \log N})$ without any prior knowledge about the maximum delay. 

In this setting, all the reward vectors $\bm{r}_i(t)$'s are decided at the beginning of the game, and  the environment cannot change these reward vectors after some actions chosen by the algorithm. \cite{Bianchi2018Nonstochastic} propose a randomized method called CLW to divide the game into many rounds, and only use partial feedback from each round. 
The randomness makes sure that in expectation, the reward in each time slot (in the same round) is the same, no matter whether it is counted into the algorithm or not. Only using partial feedback allows one to ignore 
the differences between $L_i(t)$ and $M_i(t)$. 
This means that one can use the EXP3 algorithm directly with the partial rewards as input. Based on this design, the CLW algorithm achieves an $O(\sqrt{NT\log N})$ regret upper bound. 

Our algorithm in this case, Adaptive Round-Size CLW (ARS-CLW), is shown in Algorithm \ref{Algorithm_Ext} in details, which extends the CLW algorithm  to handle the case when the reward interval size is unknown. The ARS-CLW algorithm  does not need $d$ as an input. Instead, it uses a function $h: \mathbb{N}_+ \to \mathbb{N}_+$ to estimate the value of $d$. 
Precisely, at the beginning, we set $T = T^{(1)}$, and guess $d^{(1)} = h(T^{(1)})$. If the game does not stop at $T^{(1)}$, then after $T^{(1)}$ we set $T = T^{(2)} = 2T^{(1)}$ and guess $d^{(k)} = h(T^{(2)})$. We call $[T^{(k-1)}, T^{(k)}]$ as the $k$-th phase in the algorithm (define $T^{(0)} = 0$). To make the estimated $d^{(k)}$ increasing, we constraint that $h$ must be an increasing function.  

Compare to the CLW algorithm in \cite{Bianchi2018Nonstochastic}, in each phase $k$ of our ARS-CLW policy, we modify the method to decide whether $t \in \mathcal{U}^{(k)}$, i.e.,  whether to end the current round in this time step (lines $10$-$13$ in Algorithm \ref{Algorithm_Ext}) (here $\mathcal{U}^{(k)}$ denotes the set of round-ending times).  
This modification is introduced to make sure that during each phase $k$, $\forall t, \Pr[t\in \mathcal{U}^{(k)}] = p$ for some constant $p$. 
This will allow us to use the inequalities of $p$ directly from the analysis in  \cite{Bianchi2018Nonstochastic}, and does not affect the regret bound of CLW.\footnote{If we only have $\Pr[t\in \mathcal{U}^{(k)}] \leq p$, when deriving bounds for the regret, we need to ensure that the probability is multiplied with positive values and this can be challenging.} 
Note that since we maintain a constant $\Pr[t\in \mathcal{U}^{(k)}]$ in phase $k$, the expected round size remains a constant in phase $k$. As a result, we cannot imitate ARS-UCB and ARS-EXP3 by increasing the round size continuously in ARS-CLW. 
Thus, our modified algorithm cannot change the size of the rounds frequently. That is why we choose to divide the game into several phases, and treat each phase as an independent game with unique $T$ and $d$.

\begin{algorithm}[t]
    \centering
    \caption{Adaptive Round-Size CLW Algorithm}\label{Algorithm_Ext}
    \begin{algorithmic}[1]
    \STATE \textbf{Input: } $h$, $T^{(1)}$, $k = 1$.
    \WHILE {$t \le T$}
    \STATE $d^{(k)} = h(T^{(k)})$, $\gamma^{(k)} = \sqrt{2d^{(k)} N\log N \over T^{(k)}+d^{(k)}}$, $q^{(k)} = {1\over 2d^{(k)}}$.
    \STATE Set $p^{(k)}$ to be a uniform distribution on $\{1,\cdots,N\}$.
    \STATE Draw $a^{(k)}(t-1) \sim p$.
    \STATE Generate $2d^{(k)} - 1$ i.i.d Bernoulli random variables $B^{(k)}_t,\cdots,B^{(k)}_{t+2d^{(k)}-2}$ with parameter $q^{(k)} $.
    \WHILE {$t < T^{(k)} $}
    \STATE If $t-1 \in \mathcal{U}^{(k)}$, then pick $a^{(k)}(t) \sim p$. Otherwise $a^{(k)}(t) = a^{(k)}(t-1)$.
    \STATE Generate Bernoulli random variable $B^{(k)}_{t+2d^{(k)} -1}$ with parameter $q^{(k)} $.
    \IF {$B^{(k)}_t = 1$, $B^{(k)}_{t+1} = \cdots = B^{(k)}_{t+2d^{(k)} -1} = 0$}
    \STATE Set $t \in \mathcal{U}^{(k)}$.
    \STATE Update $p$ using EXP3 policy by pulling arm $a(t)$ and obtain reward ${1\over 2d^{(k)} }\sum_{\tau=t-d^{(k)} +1}^t Z(\tau)$.
    \ENDIF
    \STATE $t\gets t+1$
    \ENDWHILE
    \STATE $k\gets k+1$
    \STATE $T^{(k)}  \gets 2 T^{(k-1)}$
    \ENDWHILE
    \end{algorithmic}
\end{algorithm}

Now we present the regret upper bound of ARS-CLW:

\begin{theorem}\label{Theorem_Ext}
	If $h: \mathbb{N}_+ \to \mathbb{N}_+$ is an increasing function such that $h(T) < T$ holds for any $T$, then Algorithm \ref{Algorithm_Ext} achieves: 
	\begin{equation*}Reg(T) = O(\sqrt{Th(2T)N\log N} + h^{-1}(d) ). \end{equation*}
\end{theorem}

We now compare the regret upper bounds in Theorem \ref{Theorem_EXP3} and \ref{Theorem_Ext}. 
The first term is almost the same. Notice that $g(k) = k^{\beta}$ implies $K = \Theta(T^{1\over \beta+1})$. Thus, the largest round size $g(K) = \Theta(T^{\beta\over \beta+1})$. If we set $h(T) = \Theta(T^{\beta\over \beta+1})$,  the largest round size is the same, implying that  the first terms in the two theorems are the same. 
The last term in Theorem \ref{Theorem_Ext} comes from the time steps during which we do not guess $d$ correctly, i.e., $d^{(k)} < d$ in the algorithm. This is different from the last term in Theorem \ref{Theorem_EXP3}, which comes from the bias in observations. 
In Algorithm \ref{Algorithm_Ext}, we use the partial information to avoid bias. As a result, such term does not exist in Algorithm \ref{Algorithm_Ext}.
Since the last term does not depend on $T$, our regret bound can be arbitrary close to the regret lower bound given in \cite{Bianchi2018Nonstochastic}, without using any prior information on $d$. 
In exchange, the increasing rate of the round size needs to be small, which causes a large regret before we have $d^{(k)} \ge d$.

\begin{proof}[Proof of Theorem \ref{Theorem_Ext}]:  
Let $k_0$ be the first round that $d^{(k_0)} \ge d$. Then, $k_0$ satisfies $h(T^{(k_0-1)}) < d$. Since $h$ in an increasing function,  $h^{-1}$ is also increasing. Thus, $T^{(k_0-1)} < h^{-1}(d)$, which implies $T^{(k_0)} < 2h^{-1}(d)$. 

We now ignore the phases  before $k_0$, as these time steps will have an additional regret of at most $2h^{-1}(d)$. 
Consider the phases with $d^{(k)} \ge d$. Notice that from Corollary 4 in \cite{Bianchi2018Nonstochastic},
each phase $k$ has a regret $O(\sqrt{d^{(k)}t^{(k)}N\log N})$, where $t^{(k)} = {T^{(k)} \over 2}$ is the number of time steps in this phase.
Then,  the total regret in these phases is upper bounded by: 
\begin{eqnarray*}
\sum_{k=k_0}^K O\left(\sqrt{d^{(k)}t^{(k)}N\log N}\right) = O\left(\sqrt{d^{(K)}t^{(K)}N\log N}\right), 
\end{eqnarray*} 
where $K$ is the last phase number.

From the description of Algorithm \ref{Algorithm_Ext}, we have that $t^{(K)} \le T$ and $d^{(K)} \le h(2T)$. Thus,  the total regret of Algorithm \ref{Algorithm_Ext} satisfies $Reg(T) \le O(\sqrt{Th(2T)N\log N} + h^{-1}(d))$.  
\end{proof}



%

\end{document}